\documentclass[accepted]{uai2025} 
                        
\usepackage[american]{babel}

\usepackage{natbib} 
\usepackage{mathtools} 
\usepackage{booktabs} 
\usepackage{tikz} 
\usepackage{comment}				
\usepackage{amsthm}					
\usepackage{algorithm, algpseudocode}	
\usepackage{bbm}					
\usepackage{amssymb}
\usepackage{tabularray} 				
\usepackage{xcolor}
\usepackage{graphicx}
\usepackage{natbib}  
\usepackage{xurl}

\newcommand{\argmin}[1]{\underset{#1}{\text{argmin }}}
\newcommand{\vertiii}[1]{{\left\vert\kern-0.25ex\left\vert\kern-0.25ex\left\vert #1 
    \right\vert\kern-0.25ex\right\vert\kern-0.25ex\right\vert}}

\theoremstyle{plain}
\newtheorem{theorem}{Theorem}
\newtheorem{lemma}{Lemma}
\newtheorem{assumption}{Assumption}
\newtheorem{prop}{Proposition}

\makeatletter
\newtheorem*{rep@theorem}{\rep@title}
\newcommand{\newreptheorem}[2]{%
\newenvironment{rep#1}[1]{%
 \def\rep@title{#2 \ref{##1}}%
 \begin{rep@theorem}}%
 {\end{rep@theorem}}}
\makeatother

\newreptheorem{theorem}{Theorem}
\newreptheorem{assumption}{Assumption}
\newreptheorem{lemma}{Lemma}

\title{Collaborative Prediction: To Join or To Disjoin Datasets }

\author[1]{Kyung Rok Kim\thanks{Corresponding to: kkrok@unc.edu, guanting@unc.edu}}
\author[2]{Yansong Wang}
\author[3]{Xiaocheng Li}
\author[1]{Guanting Chen$^*$}

\affil[1]{
    University of North Carolina at Chapel Hill
}
\affil[2]{
    University of Science and Technology of China
}
\affil[3]{
    Imperial College London
}

\usepackage{titling}
\usepackage{fancyhdr}
\begin{document}
\maketitle

\thispagestyle{fancy}
\pagestyle{plain}

\begin{abstract}
  With the recent rise of generative Artificial Intelligence (AI), the need of selecting high-quality dataset to improve machine learning models has garnered increasing attention. However, some part of this topic remains underexplored, even for simple prediction models. In this work, we study the problem of developing practical algorithms that select appropriate dataset to minimize population loss of our prediction model with high probability. Broadly speaking, we investigate when datasets from different sources can be effectively merged to enhance the predictive model's performance, and propose a practical algorithm with theoretical guarantees. By leveraging an oracle inequality and data-driven estimators, the algorithm reduces population loss with high probability. Numerical experiments demonstrate its effectiveness in both standard linear regression and broader machine learning applications. Code is available at \url{https://github.com/kkrokii/collaborative_prediction}. 
\end{abstract}

\section{INTRODUCTION}\label{sec:intro}
In the era of big data and AI, it is widely believed that more data will always help downstream performance even if the data could potentially introduce noise or even adversarial effects. However, such an approach is applicable to machine learning models that are not susceptible to overfitting as in \cite{breiman2001random, chen2016xgboost}. Meanwhile, classical methods such as linear regression and classification still serve as alternatives for prediction tasks and are still widely used in AI-related pipelines. For instance, linear probing is frequently used for large language models and other machine learning models \citep{alain2016understanding, zhang2022probing}. With the rise of generative AI, there has been a shift in the approach to better utilize data. For example, the fine-tuning and alignment stages \citep{ouyang2022training} increasingly emphasize the need for higher-quality datasets to achieve better performance. 

In this work, we investigate methods for selecting higher-quality data to enhance predictive performance. As modern AI architectures become increasingly complex, systematically studying the impact of improved data selection on performance remains a challenging task. Henceforth, we aim to study the data refinement problem for linear regression and classification problems, which are well studied in the theoretical statistics/machine learning community, while also relevant with current development of AI.

Combining multiple datasets to train machine learning models is a common practice as increasing the number of samples usually leads to an improvement of performance of the model. However, it is not always granted that a model trained on diverse datasets would perform better, compared to separate models fitted on each dataset respectively. As datasets may or may not be similar to each other, this practice raises the question of when they can be safely merged.

Given a collection of datasets, a natural way to deal with the problem would be to combine only the datasets that are relevant. By sharing information across several datasets, a model could benefit from a larger sample size. Meanwhile, the model can enhance its overall performance and provide a more accurate forecast on each task if the underlying datasets are related. This procedure of joint training among selectively chosen datasets is referred to as collaborative prediction. In collaborative prediction, the collection of assorted datasets is divided into several partitions and the same model is trained on each partition.

Research on multi-task learning \citep{evgeniou2004} and data integration \citep{merkouris2010} also aims to address this problem by utilizing information from other datasets. The majority of the studies on multi-task learning are centered around regularization of tasks. Various studies present different objective functions, optimization methods, and convergence analyses. Although optimization methods are proved to converge and have been empirically demonstrated to perform well on numerical experiments, the formulation of the objectives or regularization is less grounded by theory. Table \ref{table:vs_mtl} showcases the performance advantage of our algorithm over the regression-based multi-task learning algorithms. We also note that our problem setting differs from federated learning \citep{kairouz2021advances}, as we do not consider issues of data heterogeneity or exchangeability constraints. For a discussion of related work, we refer readers to Appendix~\ref{sec:related_work}.

\begin{table}[t]
\caption{Out-of-Sample Error Reduction of Our Algorithm vs Other Regression Based Multi-task Algorithms \citep{murugesan2017co}. Notice that these multi-task settings are not designed to specifically reduce the out-of-sample error.}
    \label{table:vs_mtl}
    \centering
    \begin{tabular}{@{}l|cccc@{}}
        \toprule
        \textbf{Datasets} & \textbf{ RSS} & \textbf{SSF} & \textbf{CFGSF}\\
        \midrule
        \textbf{Multi-task Learning} &$12\%$& $37\%$ &$18\%$\\
        \textbf{Our Algorithm} & $87\%$& $64\%$ & $71\%$  \\
        \bottomrule
    \end{tabular}
\end{table}

In this context, we propose an algorithm with provable guarantees to determine whether to combine datasets for predictive models. Our approach leverages data-driven estimations and high-probability concentration bounds. The proposed algorithm is highly versatile, applicable to both simple predictive models, such as linear regression, and more complex deep neural network architectures.

The contributions of this study are threefold:  

\begin{itemize}  
    \item First, we investigate the problem of combining multiple datasets for predictive modeling. We characterize the limitations of different models and provide theoretical insights.  
    \item Second, we establish theoretical conditions for determining whether to combine datasets, ensuring a high-probability reduction in population loss, and develop practical algorithms applicable to predictive models.  
    \item Third, we empirically demonstrate the effectiveness of our algorithm across various real-world datasets. Additionally, we highlight its applicability to both classical and modern predictive models, such as neural networks, and show its effectiveness in improving performance.  
\end{itemize}

\section{Collaborative Prediction}\label{sec_collaborative_prediction}
Let us begin with a motivating example. Denote by $\mathcal{X}$ the space of features and $\mathcal{Y}$ the space of predictors. Suppose we are given two datasets,  
$D_1 = \{(x^{(1)}_i, y^{(1)}_i)\}_{i=1}^{n_1}$ and $D_2 = \{(x^{(2)}_i, y^{(2)}_i)\}_{i=1}^{n_2}$,  
drawn from distributions $P_1$ and $P_2$, respectively. Our goal is to train a predictive model $f_{\theta}(\cdot)$ from the hypothesis class  
$\mathcal{H} = \{ f_{\theta} \mid f_{\theta}(\cdot): \mathcal{X} \to \mathcal{Y}, \theta \in \Theta \}$,  
which may range from simple linear functions to complex deep neural networks.
One approach is to train different models for each dataset, that is, we get $f_{\theta_1}$ after training on dataset $D_1$, and $f_{\theta_2}$ after training on dataset $D_2$. Another approach is to train the model $f_{\theta_c}$ on combined dataset $D_1\cup D_2$. Given a loss function $\ell:\mathcal{Y}\times\mathcal{Y} \to \mathbb{R}$, the decision to combine the datasets or use them separately depends on comparing the population loss of the combined model with that of the models trained on individual datasets. That is, we need to compare 
\begin{align*}
&R(f_{\theta_c},f_{\theta_c})\\ 
&:= \mathbb{E}_{(x,y) \sim P_1}[\ell(f_{\theta_c}(x), y)] + \mathbb{E}_{(x,y) \sim P_2}[\ell(f_{\theta_c}(x), y)]\    
\end{align*}
with
 \begin{align*}
&R(f_{\theta_1},f_{\theta_2})\\ 
&:= \mathbb{E}_{(x,y) \sim P_1}[\ell(f_{\theta_1}(x), y)] + \mathbb{E}_{(x,y) \sim P_2}[\ell(f_{\theta_2}(x), y)].     
\end{align*}
In practice, we only have access to the empirical loss. One simple approach is to compare the sample loss defined as 
\begin{align*}
    \hat{R}(f_{\theta_1}, f_{\theta_2}) &= \frac{1}{n_1}\sum_{i=1}^{n_1}\ell(f_{\theta_1}(x_i^{(1)}),y_i^{(1)})\\ 
    &+ \frac{1}{n_2}\sum_{i=1}^{n_2}\ell(f_{\theta_2}(x_i^{(2)}),y_i^{(2)}).
\end{align*}

\begin{table}[t]
\caption{Accuracy of Algorithm \ref{alg2} and Direct Comparison.}
\label{table:motivation}
    \centering
    \begin{tabular}{@{}c|c|c|c@{}}
        \toprule
        $d$ & \textbf{Merge?} & \textbf{Our Algorithm} & \textbf{Direct Comparison}\\
        \midrule
        $0$     & \text{Yes}    & $82.7\%$ & $17.3\%$\\
        $0.1$   & \text{Yes}    & $79.4\%$ & $20.5\%$  \\
        $0.3$   & \text{No}     & $73.1\%$ & $80.5\%$  \\
        \bottomrule
    \end{tabular}
\end{table}

However, this could fall short in some simple settings. In Table \ref{table:motivation}, we draw datasets from $P_1$ and $P_2$ where $x \sim N(0,I_p)$, $\epsilon \sim N(0,1)$, $y = x^\top \beta + \epsilon$ for some $\beta$ for each $P_1$ and $P_2$, and we control the distance between the two distributions by setting different values of $d = \frac{1}{\sqrt{p}}||\beta_1 - \beta_2||$. The finding is that directly comparing the empirical loss (the column \textsf{Direct} in Table \ref{table1}) would result in false decisions 80\% of the time, especially under the scenario that $d$ is relatively small. However, our proposed approach performs more consistently. 

It is also theoretically challenging to provide a tight theoretical guarantee that supports the comparison of the empirical loss directly. From generalization bound result, the excess risk is bounded by 
\begin{align*}
    |\hat{R}(f_{\theta_1},f_{\theta_2}) - R(f_{\theta_1},f_{\theta_2})| \leq O\left(R_S(\mathcal{H})\right) + O\left(\sqrt{\frac{\log(1/\delta)}{n}}\right),
\end{align*}
where $R_S(\mathcal{H})$ is the empirical Rademacher complexity. By re-arranging the terms and use the upper bound for $R(f_{\theta_c},f_{\theta_c})$ and lower bound for $R(f_{\theta_1},f_{\theta_2})$, we have $R(f_{\theta_c},f_{\theta_c}) \leq R(f_{\theta_1},f_{\theta_2})$ if
\begin{align*}
    &\hat{R}(f_{\theta_c},f_{\theta_c}) + O\left(R_S(\mathcal{H})\right) + O\left(\sqrt{\frac{\log(1/\delta)}{n}}\right) \\
    \leq &\hat{R}(f_{\theta_1},f_{\theta_2}) - O\left(R_S(\mathcal{H})\right) -O\left(\sqrt{\frac{\log(1/\delta)}{n}}\right).
\end{align*}
The inequality above implies that to estimate the population loss from the empirical loss, we must quantify $R_S(\mathcal{H})$ (or VC dimension), which is challenging for many predictive models. Moreover, the constants (denoted by the big-$O$ notation) in the standard generalization bounds often rely on additional assumptions on the loss function and distributions, such as boundedness or sub-Gaussian tail behavior, imposing further restrictions.

To summarize, the reasons not to compare the empirical loss directly are, first, it sometimes performs badly in practice; second, it is non-trivial to give tight theoretical guarantees backing this approach. Therefore, we narrow our focus to linear regression and classification for the following reasons:  

\begin{itemize}  
    \item First, these models are relatively simple and well studied in the statistical learning literature, allowing us to derive insights that may be generalized to practical scenarios.  
    \item Second, they serve as fundamental building blocks in modern machine learning pipelines. We anticipate that some of the inherent properties of the proposed algorithm may extend to more complex machine learning models.  
\end{itemize}  

Motivated by this, we develop theoretical guarantees for population loss comparison in simpler models and propose algorithms applicable to both fundamental tasks and modern machine learning applications.

\subsection{Formulation}
We now formally introduce the collaborative prediction problem. Let $\{D_k\}_{k=1}^K$ be a collection of datasets, where each dataset consists of pairs $(x^{(k)}, y^{(k)})$ sampled from $\mathcal{X} \times \mathcal{Y}$ according to some probability distribution $\{P_k\}_{k=1}^K$ over the joint space. Let $\mathcal{A}(\cdot)$ denote an algorithm that takes a dataset as input and outputs a predictive model. Specifically, when trained on a dataset $D$, the algorithm produces a function $\mathcal{A}(D): \mathcal{X} \to \mathcal{Y}$.  

To evaluate the model performance, we employ a loss function $\ell: \mathcal{Y} \times \mathcal{Y} \to [0,1]$. The population loss of the trained model is defined as  
\begin{equation}  
L(D, P) = \mathbb{E}_{(x,y) \sim P} [\ell(\mathcal{A}(D)(x), y)],  
\end{equation}  
where $P$ represents the true underlying distribution.

Suppose the algorithm is trained on each dataset $D_k$ and predicts models separately. The loss incurred by fitting individual models is the sum of losses on each dataset $\sum_{k=1}^K L(D_k, P_k)$. If the algorithm is instead trained on all the datasets $D_{all} = \cup_{k=1}^K D_k$, the loss incurred by fitting a unified model is $\sum_{k=1}^K L(D_{all}, P_k)$.
Generally speaking, merging all the data would improve the \textit{estimation error} (or variance) as the number of samples on which the algorithm is trained has increased, but it would deteriorate the \textit{approximation error} (or bias) as the same model is applied to all the datasets that are possibly different. The problem of interest is to find a partition $\mathcal{S}$ of the collection of datasets $\{D_k\}_{k=1}^K$ such that the algorithm, when trained on each partition, minimizes the total loss. Formally, a partition $S$ is defined over the index set $\{1,2,\dots, K\}$ as a collection of disjoint subsets $\{I_i\}_{i=1}^N$, where each $I_i$ satisfies $I_i \subset \{1,2,\dots, K\}$ and $\cup_{i=1}^N I_i = \{1,2,\dots, K\}.$
We denote the partition cell to which an index $k$ belongs by $S(k)$, which means $S(k) = i$ if $k \in I_i$. Our objective is to find a partition $S \in \mathcal{S}$ that minimizes the total population loss:  
\begin{equation*}  
\min_{S\in \mathcal{S}} \sum_{k=1}^K L(\cup_{j: S(j) = S(k)} D_j, P_k).  
\end{equation*}  

\subsection{Collaborative Prediction for Regression}
In this section, we establish theoretical results and data-driven algorithms on the collaborative prediction problem for regression.  We first propose theoretical results on combining \textit{two} datasets, then we progressively develop practical algorithms, extending from the combination of two datasets to multiple datasets.

\subsubsection{Criterion for Combining Datasets for Linear Regression}
Suppose a dataset $D_k = \{( x^{(k)}_i, y^{(k)}_i )\}_{i=1}^{n_k}$ is generated from a distribution $P_k$ on $\mathcal{X} \times \mathcal{Y}$ for $k=1, \cdots, K$, where $\mathcal{X} = \mathbb{R}^{p}$, $\mathcal{Y} = \mathbb{R}$, and $K=2$ temporarily. We have the following standard assumption.
\begin{assumption} \label{assumption1}
    For all $k \leq K$, we assume that $x_i^{(k)}$ follows the marginal distribution $P_{\mathcal{X}}^{(k)}$, and there exists a parameter vector $\beta^{(k)}$ such that, for all $i \leq n_k$, the response variable satisfies $y_i^{(k)} = x_i^{(k)\top} \beta^{(k)} + \epsilon_i^{(k)}$, 
where $\epsilon_i^{(k)} \sim N(0, \sigma^2)$ with $\sigma^2 > 0$. Additionally, the covariates $x_i^{(k)}$ and the noise $\epsilon_i^{(k)}$ are independent.
\end{assumption}
With this assumption, for a cleaner notation, let us condense $D_k$ more compactly as $(X^{(k)}, y^{(k)}) = ([x^{(k)}_1, \cdots, x^{(k)}_{n_k}]^\top, [y^{(k)}_1, \cdots, y^{(k)}_{n_k}]^\top) \in \mathbb{R}^{n_k \times p} \times \mathbb{R}^{n_k}$, and a merged dataset $D_1 \cup D_2$ as $(X^{(c)}, y^{(c)}) = ([X^{(1)\top}, X^{(2)\top}]^\top, [y^{(1)\top}, y^{(2)\top}]^\top)$. With this setting, we fit linear regression models on $D_1$ and $D_2$. 

The minimizer of the empirical loss is the ordinary least squares (OLS) estimator. The OLS estimator fitted from a single dataset $D_k$ can be denoted by 
\begin{align*}
    \hat{\beta}^{(k)} = (X^{(k)\top} X^{(k)})^{-1} X^{(k)\top} y^{(k)},
\end{align*}
and the estimator fitted from the combined dataset $D_1 \cup D_2$ is
\begin{align*}
    \hat{\beta}^{(c)} = (X^{(c)\top} X^{(c)})^{-1} X^{(c)\top} y^{(c)}.
\end{align*}

Now we characterize the population loss. In the setting of linear regression, the loss $\ell$ is the $L^2$ loss and the population loss is referred to as the out-of-sample error (OSE). When $(x,y)$ is sampled from $P$, we can denote the OSE of the regressor $\hat{f}(x) = x^{\top}\hat{\beta}$ by 
\begin{align*}
    \text{OSE}(\hat{\beta}, P) = \mathbb{E}_{(x,y)\sim P}\left[\left(y - x^{\top}\hat{\beta}\right)^2\right].
\end{align*}
With this setting, the comparison of population loss can be denoted by
\begin{align}\label{eq1}
\sum_{k=1}^2 \text{OSE}(\hat{\beta}^{(k)}, P_k) > 
\sum_{k=1}^2 \text{OSE}(\hat{\beta}^{(c)}, P_k) .
\end{align}
However, in practice, we only observe the sample loss. Again, it is necessary to establish theoretical results that allow us to infer population loss comparisons from sample loss observations.

To have an alternative form that is easier to compute, we define
\begin{align*}
    h & : \mathbb{R} \rightarrow \mathbb{R},    &&h(x) :=  A_0 x\\
    g & : \mathbb{R}^p \times \mathbb{R}^p \rightarrow \mathbb{R}, &&g(y, z) := ||y-z||_{B_0}^2
\end{align*}
where $A_0$ and $B_0$ are constants determined by the underlying distributions. Their precise definitions are given in Theorem \ref{thm1} in Appendix \ref{thm1section}. $||\cdot||_{B_0}$  is the Mahalanobis norm so that $||v||_{B_0} = \sqrt{v^{\top} B_0 v}$. The following theorem establishes an equivalent condition for equation \eqref{eq1}, and the proof is deferred to Appendix \ref{thm1section}.

\begin{theorem} \label{thm1}
Equation \eqref{eq1} holds if and only if $h(\sigma^2) > g(\beta^{(1)}, \beta^{(2)})$.
\end{theorem}

Theorem \ref{thm1} agrees with what one would have expected with respect to the true parameters $\beta^{(1)}$ and $\beta^{(2)}$. If the two models were close so that $\lVert \beta^{(1)}-\beta^{(2)} \rVert$ for some norm $||\cdot||$ is small, then we can have a similar condition for $||\beta^{(1)}-\beta^{(2)}||_{B_0}$ by norm equivalence. The condition in the theorem indicates that merging the datasets indeed reduces the error. 

\subsubsection{Provable Algorithm for Combining Datasets}
The true parameters $\beta^{(k)}$ are unknown in practice, and we need to develop data-driven algorithms based on Theorem \ref{thm1}. By concentration and anti-concentration inequalities, with high probability we can control the deviation of the estimation of $\sigma^2$, $\beta^{(1)}$, and $\beta^{(2)}$ from their true value. 

More specifically, the estimators can be bounded in the following way with a confidence level $\delta$. They involve some distributional constants $\tilde{A}_1(\delta)$ and $\tilde{A}_2(\delta)$ which depend on $\delta$, and $D(\cdot)$, a function of the input data, all defined in detail in Appendix \ref{phi_psi_section}. Next, we define 
\begin{align*}
\phi_\delta(u,v,w) = \tilde{A_1}(\delta) w + \tilde{A_2}(\delta) \lVert D(u-v) \rVert^2
\end{align*}
and 
\begin{align*}
\psi_\delta(u,v,w) = \left\{ \sqrt{g(u,v)} + 
    T_\delta(w, \lVert D(u-v) \rVert)
    \right\}^2
\end{align*}
where $T_\delta(x,y)$ is at the scale of $O(max(x,y))$, and also depends on the confidence level $\delta$. 
Under Assumption \ref{assumption1}, the OLS estimators $\hat{\beta}^{(1)}$ and $\hat{\beta}^{(2)}$ are normally distributed. Therefore, if the parameters $\beta^{(k)}$ are replaced by their estimators $\hat{\beta}^{(k)}$ in Theorem \ref{thm1}, then $g(\hat{\beta}^{(1)}, \hat{\beta}^{(2)})$ would be close to its true value $g({\beta}^{(1)}, {\beta}^{(2)})$. $\psi_\delta$ acts as a bound of $g(\hat{\beta}^{(1)}, \hat{\beta}^{(2)})$ with a margin depending on the properties of the underlying distribution and $\delta$ as defined above. Similarly, the sample variance of the noise on the combined dataset $\hat{\sigma}^{(c) 2}$, which has the form 
$\hat{\sigma}^{(c)2} = \lVert y^{(c)} - X^{(c)} \hat{\beta}^{(c)} \rVert^2 /(n_1+n_2-p)$, would lie near its true value $\sigma^2$. Hence $\phi_\delta$, as a linear function of $w$, behaves as a surrogate of the linear function $h$. The following lemma formalizes this result, and we leave the proof in Appendix \ref{pfsection}.

\begin{lemma} \label{lemma1}
Under Assumption \ref{assumption1}, with probability at least $1-5\delta$, we have equation \eqref{eq1} holds if 
\begin{align}\label{eqn_lemma1}
\phi_\delta (\hat{\beta}^{(1)}, \hat{\beta}^{(2)}, \hat{\sigma}^{(c) 2}) \,\,\geq \,\, \psi_\delta(\hat{\beta}^{(1)}, \hat{\beta}^{(2)}, \hat{\sigma}^{(c) 2}).
\end{align}
\end{lemma}

Lemma \ref{lemma1} provides a computationally feasible approach to verify \eqref{eq1}. It utilizes the computationally feasible terms $\phi_\delta(\hat{\beta}^{(1)}, \hat{\beta}^{(2)}, \hat{\sigma}^{(c) 2})$ and $\psi_\delta(\hat{\beta}^{(1)}, \hat{\beta}^{(2)}, \hat{\sigma}^{(c) 2})$ as proxies of the terms $h(\sigma^2)$ and $g(\beta^{(1)}, \beta^{(2)})$ in the oracle inequality, which are computationally infeasible. This result implies a directly implementable algorithm, detailed in Algorithm \ref{alg0}.

\begin{algorithm}[ht!] \caption{High Probability Criterion to Combine Two Datasets} \label{alg0}
\begin{algorithmic}[1]
\Statex Input: Datasets $\{D_1, D_2\}$, confidence level $\delta$
\Statex Output: A decision $\textbf{Merge}$ whether to merge the datasets or not
\State Fit linear models $\hat{\beta}^{(1)}$, $\hat{\beta}^{(2)}$, and $\hat{\beta}^{(c)}$ on $D_1$, $D_2$, and $D_1 \cup D_2$, respectively
\State Compute $\left(\phi_\delta(\hat{\beta}^{(1)}, \hat{\beta}^{(2)}, \hat{\sigma}^{(c) 2}) ,\psi_\delta(\hat{\beta}^{(1)}, \hat{\beta}^{(2)}, \hat{\sigma}^{(c) 2}) \right)$
\State Return  $\textbf{Merge} = \mathbbm{1}_{\{ \phi_\delta > \psi_\delta\}}$
\end{algorithmic}
\end{algorithm}
Proposition \ref{prop1} guarantees that with high probability, the decision made by Algorithm \ref{alg0} is valid.
\begin{prop} \label{prop1}
Under Assumption \ref{assumption1}, with probability at least $1-5\delta$, the output of Algorithm \ref{alg0} is a decision rule with minimum out-of-sample error being
\begin{align*}
\min \left\{ 
    \sum_{k=1}^2 \text{OSE}(\hat{\beta}^{(k)}, P_k),
    \sum_{k=1}^2 \text{OSE}(\hat{\beta}^{(c)}, P_k)
\right\}.
\end{align*}
\end{prop}

The high probability guarantee of Algorithm \ref{alg0} comes at the price of limited applicability. Computing the functions $\phi_\delta$ and $\psi_\delta$ requires some knowledge of the marginal distribution of $x$. Moreover, even when $\phi_\delta$ and $\psi_\delta$ can theoretically be computed under distributional assumptions, it will still pose a strong limitation because real data is very heterogeneously distributed. This motivates us to develop Algorithm \ref{alg1}, a computationally feasible algorithm.

\subsubsection{Approximation Algorithm} 
We develop Algorithm \ref{alg1}, which takes Algorithm \ref{alg0} as a subroutine and tries to approximate Algorithm \ref{alg0}'s decision boundary, and is computationally feasible. We apply Algorithm \ref{alg0} multiple times, and with a sufficient number of repetitions, Theorem \ref{thm2} in Appendix \ref{consistency_section} justifies the consistency of the approximation.

Because Algorithm \ref{alg1} is an approximation algorithm, there is no high probability guarantee. Therefore, Algorithm \ref{alg1} aims to estimate and maximize the \textit{success rate} (SR), 
meaning the output is correct compared to the estimated ground truth. From \eqref{eqn_lemma1}, for each $\delta$ the condition $\phi_\delta > \psi_\delta$ implies that our algorithm suggests a merge, and the success of the merge could be determined by whether $\text{OSE}\_\text{dif} := \sum_{k=1}^2 \text{OSE}(\hat{\beta}^{(k)}, P^{(k)}_{\mathcal{X} \times \mathcal{Y}}) -
\sum_{k=1}^2 \text{OSE}(\hat{\beta}^{(c)}, P^{(k)}_{\mathcal{X} \times \mathcal{Y}})>0$.
Henceforth, the success rate (SR) is formally defined as 
\begin{align}\label{eqn:SR}
    \text{SR}  = \mathbb{P}\left((\phi_\delta - \psi_\delta) \times \text{OSE}\_\text{dif} > 0 \right).
\end{align}

As we do not know the exact $\text{OSE}\_\text{dif}$ due to the lack of information about the distribution, we approximate it based on given data. More specifically, suppose our available dataset $D_k$ is generated from a distribution $P_k$ for $k=1,2$. A subset $D_k^{train}$ of $D_k$ is reserved for training and the rest $D_k^{out} = D_k - D_k^{train}$ is used to generate out-of-samples $\{ (\tilde{x}^{(k)}_i, \tilde{y}^{(k)}_i) \}_{i=1}^{\tilde{n}_k} $ for some out-of-sample size $\tilde{n}_k$, independent from the training data. Under these settings, $\text{OSE}\_\text{dif}$ is estimated by 
    $\widehat{\text{OSE}}\_\text{dif} := \sum_{k=1}^2 \widehat{\text{OSE}} \left(\hat{\beta}^{(k)}, D_k^{out} \right) 
- \sum_{k=1}^2 \widehat{\text{OSE}} \left(\hat{\beta}^{(c)},D_k^{out} \right)$ 
where $\widehat{\text{OSE}}$ approximates the true out-of-sample error of an estimator $\hat{\beta}$ by
\begin{align*}
\widehat{\text{OSE}} \left(\hat{\beta}, D_k^{out} \right) = \frac{1}{\tilde{n}_k} \sum_{i=1}^{\tilde{n}_k} \left( \tilde{y}^{(k)} - \tilde{x}^{(k) \top} \hat{\beta} \right)^2 .
\end{align*}
Finally, we repeatedly apply the algorithm several times and estimate $\text{SR}$ by its average success rate $\widehat{\text{SR}}$.

We then present Algorithm \ref{alg1} with SR as the optimization target. Algorithm \ref{alg1} treats $\delta$ as a hyperparameter, whose optimal value can be finetuned by maximizing the estimated success rate $\widehat{\text{SR}}$. We follow \cite{arlot2009} and employ grid search to find the optimal value. As Proposition \ref{prop1} holds in probability, we evaluate the SR by repeating Algorithm \ref{alg0} several times. The full procedure is described in detail in Algorithm \ref{alg1}.

\begin{algorithm}[ht!] \caption{Merging Two Datasets with Optimal Confidence Level} \label{alg1}
\begin{algorithmic}[1]
\Statex Input: Datasets $\{D_1, D_2\}$, hyperparameter range $[\alpha_{min}, \alpha_{max}]$, grid search window size $\eta$, accuracy threshold $\lambda$
\Statex Output: Decision $\mathbbm{1}_{merge}$ whether to merge the datasets or not, proxy accuracy $proxy\_acc$
\Statex \textbf{Part1: Tuning $\alpha$}
\For{$\alpha = \alpha_{min}, \alpha_{min} + \eta, \cdots, \alpha_{max} $}
    \State Initialize $correct_\alpha$ = 0          \label{alg1:ref1}
    \For{$j = 1, 2, \cdots, max\_iterations $}
        \State Sample $(X^{(k)}, y^{(k)})$ from $D_k$ $(k=1,2)$ 
        \State Run Algorithm \ref{alg0} on $\{(X^{(k)}, y^{(k)})\}_{k=1}^2$ with $\log{\frac{1}{\delta}}$ replaced by $\alpha$ to obtain $(\phi_\alpha, \psi_\alpha)$
        \State Generate out-of-samples $(\tilde{X}^{(k)}, \tilde{y}^{(k)})$ from $D_k$ for $k=1,2$
        \State Compute $\widehat{\text{OSE}}\_\text{dif}$
        \State $correct_\alpha \gets correct_\alpha + \mathbbm{1}_{(\phi_\alpha - \psi_\alpha) \times OSE\_\text{dif} > 0}$
    \EndFor     \label{alg1:ref2}
\EndFor
\State Choose $\alpha_{opt} = \text{argmax}_\alpha correct_\alpha$
\Statex \textbf{Part2: Analysis of the data}
\State Repeat \ref{alg1:ref1}-\ref{alg1:ref2} with $\alpha = \alpha_{opt}$
\State Return $proxy\_acc = correct_{\alpha_{opt}} / max\_iterations $ and $\textbf{Merge} = \mathbbm{1}_{proxy\_acc > \lambda}$
\end{algorithmic}
\end{algorithm}

Algorithm \ref{alg1} provides a self-contained process for making a decision in practice. It uses only a portion of the given dataset to train models. The rest are used as out-of-samples to compute $\widehat{OSE}$. In Algorithm \ref{alg1}, Algorithm \ref{alg0} is applied with different values of $\alpha$, and the optimal value is chosen based on its performance. After tuning $\alpha$, the algorithm can be employed to make the actual decision. Due to the page limit, we put Theorem \ref{thm2} which justifies the consistency of the approximation, in Appendix \ref{consistency_section}.

\subsubsection{Extension to Multiple Datasets}
Algorithm \ref{alg1} is based on Proposition \ref{prop1}, which provides a provable guarantee for a pairwise decision. To generalize to the cases where multiple datasets are given, Algorithm \ref{alg1} is applied iteratively to choose which datasets to merge. Here, we compare multiple datasets and form clusters by adopting a greedy algorithm. Specifically, given a dataset that has not yet been assigned any cluster, we compare it with other datasets that also have not been clustered. We choose the one with the largest boost in performance when merged with the given dataset, and decide to combine the two if the elevated performance exceeds a certain threshold. The comparison is repeated until no other dataset can be merged to the current cluster. Details of the algorithm are presented in Algorithm \ref{alg2}.

\begin{algorithm}[ht!] \caption{Merging Multiple Datasets with Greedy Algorithm} \label{alg2}
\begin{algorithmic}[1]
    \Statex Input: A collection of datasets $\{D_k\}_{k=1}^K$, accuracy threshold $\lambda$ 
    \Statex Output: Cluster of each dataset $\{c_k\}_{k=1}^K$
    \While{$\exists D_k$ without a cluster}
        \State Assign a new cluster $c_k$ to $D_k$
        \State For every $D_j$ without a cluster, choose $j$ that maximizes $proxy\_acc_{(k, j)}$ by Algorithm \ref{alg1}                                  \label{alg2:ref1}
        \While{$proxy\_acc_{(k,j)} > \lambda $}
            \State Assign $c_j$ the same cluster as $c_k$, and merge $D_k$ and $D_j$
            \State Repeat \ref{alg2:ref1}
        \EndWhile
    \EndWhile
    \State Return $\{c_k \}_{k=1}^K$
\end{algorithmic}
\end{algorithm}

We note that the greedy algorithm is chosen for simplicity and practicability. It is straightforward to understand, and it is guaranteed to return a non-increasing error. Moreover, Algorithm \ref{alg2} has a quadratic dependence on the number of datasets. To the best of our knowledge, there is no standard clustering method faster than the greedy algorithm among those that are suitable for our setting. This is due to the absence of a universal feature space in which datasets are embedded, which necessitates pairwise comparison to precede a clustering method. Further complexity analysis, including running time report, is provided in Appendix \ref{section:complexity}.

\subsection{Collaborative Prediction for Classification} \label{section:classification_summary}
Similar arguments can be made on classification tasks as in regression. We employ logistic regression model to make decisions and assess its population error with cross entropy loss, with details provided in Appendix \ref{classificaion_section}. In a broad sense, the decision should be made by comparing errors, where the error now refers to classification error rather than the regression error. Once again, we want to represent the population loss in terms of $\beta$, the linear parameters related to the covariates in the generating distribution. Under moderate assumptions, it can be proved that combining the datasets gives smaller population error bound when
\begin{align*}
    \Psi(\beta^{(1)}, \beta^{(2)}) \leq \Phi(\beta^{(1)}) + \Phi(\beta^{(2)}),
\end{align*} 
where the precise definition of $\Psi$ and $\Phi$ can be found in Appendix \ref{section:classification_single} and \ref{section:classification_combined}.  The major difference between the two functions, as described in Theorem \ref{thm:classification_final} in Appendix \ref{section:classification:conclusion}, is that $\Psi$ involves a term of $O\left(\left(\frac{1}{\sqrt{n_1}} + \frac{1}{\sqrt{n_2}}\right)\lVert \beta^{(1)}-\beta^{(2)} \rVert\right)$, while $\Phi$ contains $O\left(\frac{\beta}{\sqrt{n}}\right)$. Hence the error bound on the combined dataset decreases if the distance between the true parameters $\lVert \beta^{(1)} - \beta^{(2)} \rVert$ is small, which aligns with our previous conclusion on regression. Due to the page limit, we refer readers to Appendix \ref{classificaion_section} for precise statements and details.

\section{NUMERICAL EXPERIMENTS} 
In this section, we empirically show the effectiveness and broad applicability of our algorithm. The structure of the section is as follows.
\begin{itemize}
    \item In Section \ref{sec_emp_syn}, we conduct experiments in synthetic environments. The results demonstrate strong performance, aligning well with our theoretical predictions.  
    \item In Section \ref{sec_emp_real}, we evaluate our algorithm on real datasets, casting them as regression problems. It exhibits strong performance, effectively handling datasets with heterogeneous distributions.  
    \item In Section \ref{sec_emp_nn}, we extend our algorithm to neural networks. The results not only show its robustness and strong performance across different settings but also highlight its applicability to different predictive models.  
\end{itemize}
Experimental details, such as the structure of neural networks or the choice of  hyperparameters, can be found in Appendix \ref{app_nn} and \ref{section:hyperparameter_setup}.

\subsection{Synthetic Environment}\label{sec_emp_syn}
We first test our example in the synthetic environment, where we know (or can sample) the ground truth such that whether datasets should be combined can be precisely determined. Also, the extent to which the datasets are heterogeneous could be controlled by regulating the difference between the underlying distributions. By experimenting within different environment setups, we can have a comprehensive idea how and when our algorithm will be effective.

We first consider the case where two datasets are available, and further deal with general settings on subsequent sections. Each of the dataset is generated by a linear model $y^{(k)}_j = x^{\top (k)}_j \beta^{(k)} + \epsilon^{(k)}_j$ $( k=1,2$, $j=1, \cdots, n_k$, $\beta^{(k)} \in \mathbb{R}^p)$ with randomly generated covariates $x^{(k)} \sim N(\mu^{(k)}, \Sigma^{(k)})$ and Gaussian noise $\epsilon \sim N(0, \sigma^2 I)$. We select different combinations of $p$, $\mu^{(k)}$, and $\Sigma^{(k)}$, and compare the sum of the out-of-sample errors of individual models $\sum_{k=1}^2 OSE(\hat{\beta}^{(k)}, P^{(k)}_{\mathcal{X} \times \mathcal{Y}})$ with the out-of-sample error of combined model $\sum_{k=1}^2 OSE(\hat{\beta}^{(c)}, P^{(k)}_{\mathcal{X} \times \mathcal{Y}})$.

Because not many related algorithms are available, we compare Algorithm \ref{alg1} with the algorithm called Direct Comparison, which is a naive algorithm that directly compares the sample loss (described in the beginning of Section \ref{sec_collaborative_prediction}). In the synthetic environment, we know the ground truth of the decision whether or not to combine the datasets. The accuracy metric is computed by comparing the ground truth with the outputs of algorithms for evaluation.

We show results for two settings. In both environments, the datasets are sampled from different distributions. In the first setting (in Table \ref{table1}), we gradually shift $\beta^{(2)}$ by increasing $d$, where $\beta^{(2)} - \beta^{(1)} = d{\bf{1}}_p$ and ${\bf{1}}_p \in \mathbb{R}^p$ is a unit vector. In the second environment (in Table \ref{table1.2}), we make an additional change so that the covariate $x$ is sampled from a distribution with shifted mean ($\mu_1 = 0$ and $\mu_2 = \bf{1}_p$). The performance of Algorithm \ref{alg1} and Direct Comparison is summarized in Table \ref{table1} and Table \ref{table1.2}. The accuracy metric is computed by comparing the ground truth with their outputs. 

Algorithm \ref{alg1} successfully reduces out-of-sample error by making adaptive decisions on a wide variety of circumstances. Intuitively, combining the datasets for small values of $d$ would aid the performance of a model, as the underlying distributions are similar, while the datasets should be deemed as distinct and sharing information will not improve the model if $d$ is large. The ground truth decision for combining the datasets changes mostly around $c \in (0.1, 0.3)$. Algorithm \ref{alg1} intelligently sorts out the situation. For $c \leq 0.1$, combining the two datasets indeed lowers the out-of-sample error for most of the cases, and the algorithm also suggest merging the datasets. For $c \geq 0.3$, information of two distributions becomes irrelevant according to the out-of-sample errors, and Algorithm \ref{alg1} also captures this fact.

\begin{table}[t]
\caption{Accuracy of Algorithm \ref{alg1} on the Synthetic Data: $n_1 = n_2 = 50, \mu_1 = \mu_2 = 0, \Sigma_1 = \Sigma_2 = I$}
\label{table1}
    \centering
    \begin{tabular}{@{}c|c|c|c|c@{}}
        \toprule
        $p$ &   $d$ & \textbf{Merge?} & \textbf{Algorithm \ref{alg1}} & \textbf{Direct Comparison}\\
        \midrule
        $10$    & $0$   & \text{Yes}    & $82.7\%$  & $17.3\%$  \\
        $10$    & $0.1$ & \text{Yes}    & $79.4\%$  & $20.5\%$  \\
        $10$    & $0.3$ & \text{No}     & $73.1\%$  & $80.5\%$  \\
        $20$    & $0$   & \text{Yes}    & $87.4\%$  & $12.6\%$  \\
        $20$    & $0.1$ & \text{Yes}    & $80.4\%$  & $19.4\%$  \\
        $20$    & $0.3$ & \text{No}     & $70.1\%$  & $92.5\%$  \\
        \bottomrule
    \end{tabular}
\end{table}

\begin{table}[t]
\caption{Accuracy of Algorithm \ref{alg1} on the Synthetic Data: $n_1 = n_2 = 50, \mu_1 = 0, \mu_2 = {\bf{1}}_p, \Sigma_1 = \Sigma_2 = I$}
\label{table1.2}
    \centering
    \begin{tabular}{@{}c|c|c|c|c@{}}
        \toprule
        $p$ &   $d$ & \textbf{Merge?} & \textbf{Algorithm \ref{alg1}} & \textbf{Direct Comparison}\\
        \midrule
        $10$    & $0$   & \text{Yes}    & $80.6\%$  & $19.4\%$  \\
        $10$    & $0.1$ & \text{Yes}    & $73.6\%$  & $22.6\%$  \\
        $10$    & $0.3$ & \text{No}     & $98.0\%$  & $98.2\%$  \\
        $20$    & $0$   & \text{Yes}    & $84.7\%$  & $15.3\%$  \\
        $20$    & $0.1$ & \text{Yes}    & $66.6\%$  & $32.7\%$  \\
        $20$    & $0.3$ & \text{No}     & $99.1\%$  & $99.9\%$  \\
        \bottomrule
    \end{tabular}
\end{table}

\subsection{Application on Real-World Data} \label{sec_emp_real}
We now apply the proposed algorithm to real-life datasets. We demonstrate results on four datasets, which are selected from frequently cited sources in machine learning community such as the UC Irvine repository. The datasets are: 1) Demand Forecast for Optimized Inventory Planning (DFOIP) \citep{aguilar2023}, which contains records of online transactions on different items, 2) Walmart Data Analysis and Forcasting (WDAF) \citep{sahu2023}, which consists of weekly sales and factors that might impact customers, 3) Boom Bikes (BB) \citep{mishra2021}, data on the number of daily bike rentals with weather information, and 4) Productivity Prediction of Garment Employees (PPGE) \citep{siri2021}, data on productivity of workers in a company. 

We manually divide the datasets into separate parts based on certain features. The datasets can now be viewed as sampled from different conditional distributions (conditioned on the feature), and it becomes reasonable to believe that a single model can be consistently applied to these parts. Namely, DFOIP is partitioned for each item, the ten most popular of which are used for the experiment. WDAF is divided into holiday and non-holiday weeks, BB is split by weather conditions, and PPGE is divided according to day.  

We find that carefully choosing which datasets to combine can greatly improve model performance. Merging all DFOIP datasets, for instance, increases data size but may reduce accuracy, as customer behavior varies across products. Separate models offer tailored predictions, but may struggle with new items due to limited data. In this case, Algorithm \ref{alg2} provides a solution by merging datasets of similar products. 

The performance of Algorithm \ref{alg2} is illustrated in Table \ref{table2}. Note that on real data, the true out-of-sample errors are inaccessible as parameters of underlying distributions are unknown. Hence out-of-sample errors are estimated via bootstrapped samples. As Table \ref{table2} exhibits, Algorithm \ref{alg2} reduces the out-of-sample error by more than 17.48\% by aggregating pertinent datasets, compared to out-of-sample error of individual models trained independently on each dataset. The results suggest that even with simple clustering, our algorithm greatly reduces the error.

\begin{table}[t]
\caption{Performance of Algorithm \ref{alg2} Reducing $\widehat{\text{OSE}}$}
\label{table2}
    \centering
    \begin{tabular}{@{}c|c|c|c@{}}
        \toprule
                & \textbf{Individual}    & \textbf{Algorithm \ref{alg2}}   & \textbf{Reduction} \\
        \midrule
        DFOIP   & $4.14 \times 10^8$    & $1.25 \times 10^6$    & $99.7\%$  \\
        WDAF    & $8.98 \times 10^{11}$ & $7.41 \times 10^{11}$ & $17.5\%$  \\
        BB      & $2.057$               & $1.347$               & $34.5\%$  \\
        PPGE    & $5.57 \times 10^{-2}$ & $3.95 \times 10^{-2}$ & $29.1\%$  \\
        \bottomrule
    \end{tabular}
\end{table}

\subsection{Integration on Neural Networks}\label{sec_emp_nn}
We further integrate our algorithm into neural networks. Neural networks are often pre-trained on massive amount of data and then fine-tuned for specific tasks by modifying the final layer. It is well-known that the representations before the final layer contain useful information of the input. By treating these representations as input, our algorithm can effectively distinguish between similar and dissimilar representations and improve predictive performance. See Appendix \ref{app_nn} for the details of applying our algorithms in the representation space of neural networks.

Five datasets are analyzed in this section: 1) Rossmann Store Sales (RSS) \citep{florianknauer2015}, which consists of sales of different stores along with additional information on each day, 2) Store Sales Forcasting (SSF) \citep{admin2014}, weekly sales data of Walmart stores, 3) Corporación Favorita Grocery Sales Forecasting (CFGSF) \citep{favorita2017}, another retail sales data in Ecuador, 4) Seoul Bike Sharing Demand (SBSD) \citep{veerappampalayam2020}, weekly bike demand data, and 5) Metro Interstate Traffic Volume (MITV) \citep{hogue2019}, data of traffic volume in Minneapolis. Retail datasets are divided by store types. SSF and CFGSF is further split into department and item types. The last two datasets are divided based on season and weather.

We use two simple networks: multi-layer perceptron with one hidden layer (MLP1) and two hidden layers (MLP2). While more sophisticated and advanced models could definitely be used to extract representation, and those models would indeed return finer representation of data, our algorithm would be further highlighted if it performs well even on these simple models. After training the models, we take the penultimate layer's representations as new datasets. We then apply Algorithm \ref{alg2} and check whether out-of-sample error is reduced by the algorithm.

\begin{table}[t]
\caption{Algorithm \ref{alg2} on Representation with $\widehat{\text{OSE}}$}
\label{table3}
    \centering
    \begin{tabular}{@{}c|c|c|c@{}}
        \toprule
        MLP1    & Individual    & Alg. \ref{alg2}   & Reduction \\
        \midrule
        RSS     & $2.197 \times 10^3$   & $2.326 \times 10^2$   & $89.41\%$   \\ 
        SSF     & $1.381 \times 10^4$   & $4.310 \times 10^3$   & $68.79\%$   \\ 
        CFGSF   & $1.564 \times 10^1$   & $4.308$               & $72.46\%$   \\ 
        SBSD    & $1.325 \times 10^1$   & $4.830$               & $63.55\%$   \\ 
        MITV    & $2.111 \times 10^3$   & $1.388 \times 10^3$   & $72.46\%$   \\ 
        \bottomrule
    \end{tabular}
    
    \begin{tabular}{@{}c|c|c|c@{}}
        \toprule
        MLP2    & Individual    & Alg. \ref{alg2}   & Reduction \\
        \midrule
        RSS   & $1.327 \times 10^3$ & $2.760 \times 10^2$ & $79.20\%$   \\ 
        SSF   & $3.277 \times 10^6$ & $3.816 \times 10^4$ & $98.83\%$   \\ 
        CFGSF & $1.320 \times 10^2$ & $4.723 \times 10^1$ & $64.22\%$   \\ 
        SBSD  & $1.146 \times 10^1$ & $4.110$             & $63.55\%$   \\ 
        MITV  & $1.699 \times 10^3$ & $2.028 \times 10^2$ & $88.06\%$   \\ 
        \bottomrule
    \end{tabular}
\end{table}

The outcome of Algorithm \ref{alg2} is summarized in Table \ref{table3} along with the out-of-sample errors estimated by bootstrapping as before. The errors are measured on representation of the models rather than the original datasets. As can be confirmed from Table \ref{table3}, the error decreases 64.22\% at minimum. The results suggest that our algorithm is good at distinguishing and combining relevant datasets, even on these simple neural network models, suggesting the effectiveness of our algorithm for more complex models as well.

\begin{figure}[h]
\begin{center}
\vspace{.3in}
\includegraphics[scale=0.25]{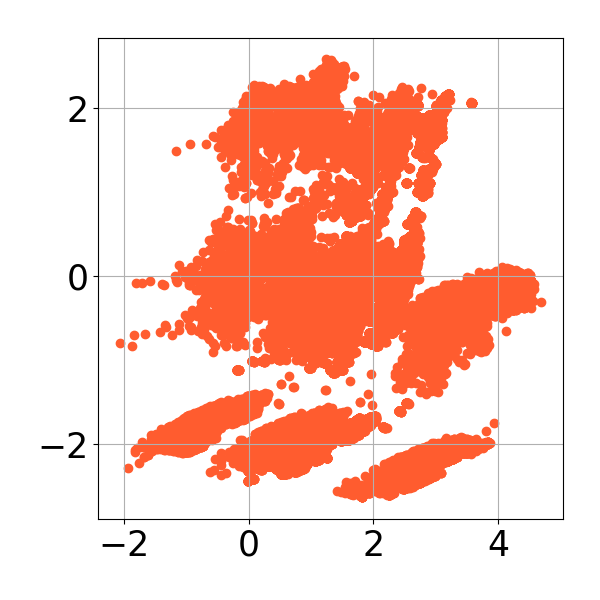}
\includegraphics[scale=0.25]{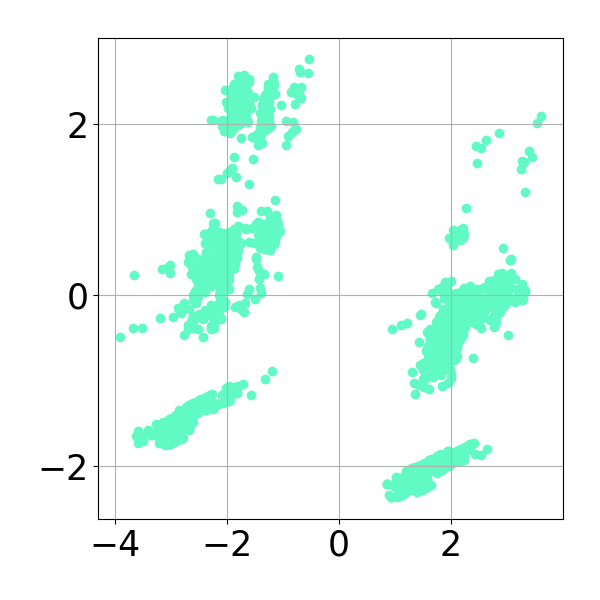}
\includegraphics[scale=0.25]{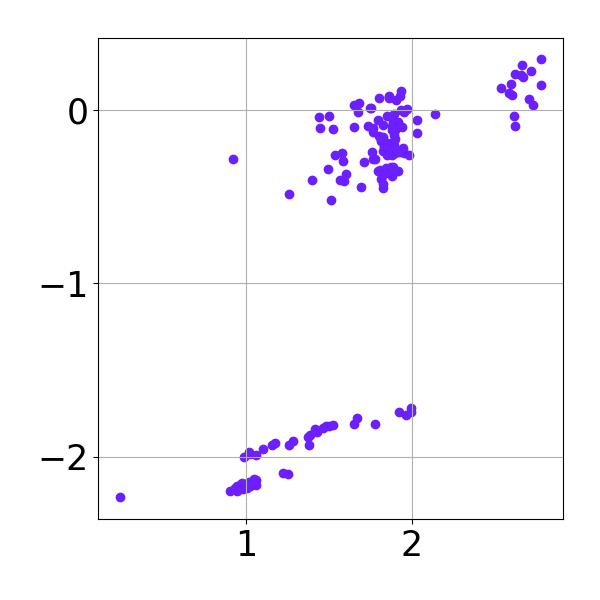}
\includegraphics[scale=0.25]{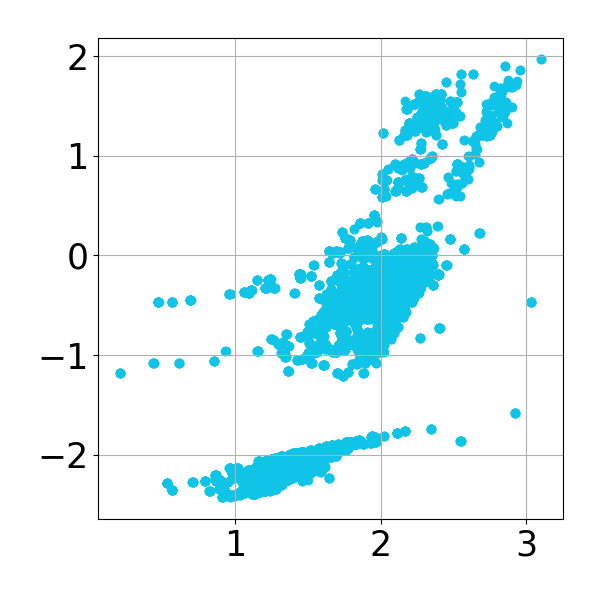}
\vspace{.3in}
\caption{Representation of MLP2 on CFGSF}
\label{fig1}
\end{center}
\end{figure}

The effectiveness of clustering the datasets can also be confirmed by visualizing the representation. Figure \ref{fig1} demonstrates representation of MLP2 projected onto $\mathbb{R}^2$ by principal component analysis. We present four collections of datasets from CFGSF, with each collection containing multiple datasets that are classified as the same cluster by Algorithm \ref{alg2}. As illustrated, each cluster includes multiple groups of representation from different datasets, which are heterogeneous and non-Gaussian. Nevertheless, the datasets are clustered to reduce the overall error. Figure \ref{fig1} shows that our algorithm successfully reduces error even when the underlying distribution of the input data does not meet the theoretical assumptions.

\section{CONCLUSION}

We present a data-driven algorithm for combining multiple datasets to share information across similar tasks with theoretical guarantee.  We empirically demonstrate that our proposed algorithm smartly combines similar datasets and reduces population loss on real data. Additionally, we showed that it can be successfully integrated into neural networks. Future works include exploring other clustering strategies with provable guarantees to improve dataset selection. It is also interesting to extend this framework to other machine learning models.

\bibliography{uai2025-template}

\newpage
\onecolumn
\title{Supplementary Material}
\emptythanks
\maketitle
\appendix

\section{Related Work}\label{sec:related_work}
\paragraph{Multi-task learning.}
Given a collection of datasets from multiple tasks, multi-task learning aims to enhance the learning ability by exploiting commonalities between the tasks. A large portion of the literature focuses on optimization algorithms and their convergence on risk minimization problem. Structural learning with linear predictors was studied by \cite{ando2005}. Without theoretical justification, \cite{evgeniou2004} took regularization into account in multi-task learning. \cite{amit2007} proposed to decompose tasks into feature clusters and task clusters to learn common characteristics among different tasks. \cite{lin2018}, \cite{kumar2012}, \cite{barzilai2015}, and \cite{murugesan2017co} also used similar ideas with different regularization tactics. \cite{zhou2016} and \cite{yao2019} suggested to seek for representative tasks to flexibly cluster tasks by allowing a single task to be related to multiple representative tasks without specifying the number of cluster in advance.

Several studies also attempted to address the question related to the performance of the estimators. \cite{lounici2009} worked on multi-task linear regression and provided a performance guarantee with high probability for a specific regularization coefficient, assuming sparsity of the tasks. \cite{liang2009} devised asymptotical criterion for when a plugin estimator should be preferred to oracle or unregularized estimator. \cite{blum2017} studied PAC framework of multi-task learning on personalized and centralized settings without making connection among similar tasks. The work that is most closely related to our study was done by \cite{solnon2012}, where they tackled multi-task kernel regression problem by estimating regularization coefficient based on the given data and providing high probability guarantees on their model.

\paragraph{Concentration and anti-concentration inequality.}
Studies on concentration inequality aim to provide a bound on the tail probability of a distribution. A wide range of study deals with sub-Gaussian random variables, see \cite{vershynin2018} and \cite{boucheron2013} for example. \cite{hanson1971} first proposed a tail bound for a quadratic form of independent random variables, and \cite{rudelson2013} later gave a modern proof. \cite{hsu2011} also dealt with quadratic forms of sub-Gaussian random vectors. Meanwhile, works on anti-concentration inequality pose a bound on the probability that a random variable resides in a small region. The concept of anti-concentration is based on Levy's concentration function (\cite{kolmogorov1958}). \cite{latala2005} provided an upper bound on the probability that a random variable falls in a small ball. \cite{dasgupta2002} presented similar type of bound on $\chi^2$ distribution.

\paragraph{Representation in neural networks.}
It is widely accepted that features from a neural network can act as representation of data. \cite{alain2016understanding} extensively investigated features of each layer of neural networks and empirically showed that deeper features predict class labels well. \cite{li2017} discovered that features of different models could converge to a set of features which serves as a low dimensional representation. \cite{athiwaratkun2015} removed the top output layer of convolutional neural networks (CNNs) and demonstrated that classification with features outperforms the original CNN. \cite{bakker2003} used outputs from the penultimate layer of a multi-layer perceptron as features for multi-task learning, and empirically proved that task clustering increases the model performance.

\paragraph{Collaborative prediction.} 
The term collaborative prediction is used in the literature in slightly different points of view. One line of research, \cite{fan2023collaborative} and \cite{zhang2024collaborative} for example, focuses on generating better predictions given a fixed set of trained models. These works either demonstrate performance gain empirically or provide theory on asymptotics. Another line of research uses the term collaborative prediction and collaborative filtering interchangeably to refer to the problem of matrix completion, as in \cite{srebro2004generalization} and \cite{yu2009large}. On the contrary, we define collaborative prediction as training an unspecified number of models on homogeneous datasets.

\section{Proofs and Details} \label{pfsection}
In this section, we prove the theorems presented in the main text and delineate details of the classification problem in Section \ref{section:classification_summary}. We make a brief remark on overparametrized regime.

We first review the notations and setups. Then the theorems are proved under common assumptions. 

\subsection{Regression}
On regression, it is possible to precisely compute out-of-sample error of a model. We first provide the expression for the error, which contains true unknown parameters. We next replace the parameters with data-driven estimators which could be computed explicitly given data. Consistency of the error is also proved under appropriate assumptions.

\subsubsection{Exact Out-of-Sample Error} \label{thm1section}
Let $\mathcal{X}$ and $\mathcal{Y}$ denote the space of features and predictors, respectively. A dataset $D_k$ is a set of data sampled from $\mathcal{X \times Y}$ by some probability distribution. Suppose $D_k = \{(x^{(k)}_i, y^{(k)}_i) \in \mathbb{R}^p \times \mathbb{R}| i=1, \cdots, n_k\}$ is generated from a distribution $P^{(k)}_{\mathcal{X} \times \mathcal{Y}}$ for $k=1, \cdots, K$. We concatenate all the samples in the same dataset as $(X^{(k)}, y^{(k)}) = \left( \begin{bmatrix} x^{(k)\top}_1 \\ \vdots \\ x^{(k)\top}_{n_k} \end{bmatrix}, \begin{bmatrix} y^{(k)}_1 \\ \vdots \\ y^{(k)}_{n_k} \end{bmatrix}\right)$, and restate Assumption \ref{assumption1} below.

\begin{repassumption}{assumption1}
    (restated) A sample $(x^{(k)},y^{(k)})$ is generated from the distribution $P^{(k)}_{\mathcal{X} \times \mathcal{Y}}$ such that 
    $ x^{(k)} \sim P_{\mathcal{X}}^{(k)}$, the marginal distribution w.r.t. $X$, and
    $y^{(k)} = x^{(k)\top} \beta^{(k)} + \epsilon^{(k)}$, where $\epsilon^{(k)} \sim N(0,\sigma^2)$ and $\sigma > 0$. Additionally, the covariates and the noise are independent.
\end{repassumption}

Let us consider the case where only two datasets $D_1$ and $D_2$ exist. We stack all the data as $(X^{(c)}, y^{(c)}) = \left( \begin{bmatrix} X^{(1)} \\ X^{(2)} \end{bmatrix}, \begin{bmatrix} y^{(1)} \\ y^{(2)} \end{bmatrix}\right)$. The (unbiased) OLS estimator on $D_k$ is $\hat{\beta}^{(k)} = (X^{(k)\top} X^{(k)})^{-1} X^{(k)\top} y^{(k)}$, the (unbiased) OLS estimator on $D_1 \cup D_2$ is $\hat{\beta}^{(c)} = (X^{(c)\top} X^{(c)})^{-1} X^{(c)\top} y^{(c)}$, and the (unbiased) estimator for the variance of the noise on $D_1 \cup D_2$ is $\hat{\sigma}^{(c) 2 } = \frac{\lVert y^{(c)} - X^{(c)} \hat{\beta}^{(c)} \rVert^2}{n_1+n_2-p}$. It is well-known that the out-of-sample error of the OLS estimator $\text{OSE}(\hat{\beta}^{(k)}, P^{(k)}_{\mathcal{X} \times \mathcal{Y}})$ is of the following form 
\begin{align*}
    \text{OSE}(\hat{\beta}^{(k)}, P^{(k)}_{\mathcal{X} \times \mathcal{Y}}) = \sigma^2 + \sigma^2 tr(E[\tilde{x}^{(k)} \tilde{x}^{(k)\top} ] E[ (X^{(k)\top} X^{(k)})^{-1} ]),
\end{align*} where $\tilde{x}^{(k)}$ is an out-of-sample generated from the same marginal distribution $P_{\mathcal{X}}^{(k)}$. For the following theoretical development, we present an alternative form of $\text{OSE}(\hat{\beta}^{(c)}, P^{(k)}_{\mathcal{X} \times \mathcal{Y}})$. 

\begin{lemma} \label{alemma1}
    Under Assumption \ref{assumption1},
    \begin{align*}
        \text{OSE}(\hat{\beta}^{(c)}, P^{(k)}_{\mathcal{X} \times \mathcal{Y}}) &= \sigma^2 
        + \sigma^2 tr( W_k E[(X^{(c)\top} X^{(c)})^{-1}] ) 
        + (\beta^{(1)} - \beta^{(2)})^\top E[Z_{3-k}^\top W_k Z_{3-k}] (\beta^{(1)} - \beta^{(2)}),
    \end{align*}
    where $W_k  = E[\tilde{x}^{(k)} \tilde{x}^{(k)\top}]$ for an out-of-sample $\tilde{x}^{(k)} \sim P^{(k)}_X$ and $Z_k = (X^{(c)\top} X^{(c)})^{-1} X^{(k)\top} X^{(k)}$ .
\end{lemma}
\textbf{Remark}. For our two datasets case, the $(3-k)$th dataset refers to the other dataset instead of $D_k$. For instance, if $k=1$ so that $\hat{\beta}^{(1)}$ is being analyzed, then $Z_{3-k}$ refers to $(X^{(c)\top} X^{(c)})^{-1} X^{(2)\top} X^{(2)}$.

\begin{proof}
By definition,
\begin{align*}
    OSE(\hat{\beta}^{(c)}, P^{(k)}_{\mathcal{X} \times \mathcal{Y}}) 
    &= E[ (\tilde{y}^{(k)} - \tilde{x}^{(k)\top} \hat{\beta}^{(c)})^2] \\
    &= E[ (\tilde{x}^{(k)\top} \beta^{(k)} + \epsilon - \tilde{x}^{(k)\top} \hat{\beta}^{(c)})^2] \\
    &= E[ (\tilde{x}^{(k)\top} \beta^{(k)} - E[\tilde{x}^{(k)\top} \hat{\beta}^{(c)} | \tilde{x}^{(k)}])^2] 
    + E[( E[\tilde{x}^{(k)\top} \hat{\beta}^{(c)} | \tilde{x}^{(k)}] - \tilde{x}^{(k)\top} \hat{\beta}^{(c)})^2] \\
    & \quad + 2E\left[ (\tilde{x}^{(k)\top} \beta^{(k)} 
    - E[\tilde{x}^{(k)\top} \hat{\beta}^{(c)} | \tilde{x}^{(k)}])( E[\tilde{x}^{(k)\top}\hat{\beta}^{(c)} | \tilde{x}^{(k)}]
    - \tilde{x}^{(k)\top} \hat{\beta}^{(c)}) \right] + \sigma^2 .
\end{align*}
Note that 
\begin{align*}
    E[\tilde{x}^{(k)\top} \hat{\beta}^{(c)} | \tilde{x}^{(k)}]
    &= E[\tilde{x}^{(k)\top} (X^{(c)\top} X^{(c)})^{-1} X^{(c)\top} y^{(c)}  | \tilde{x}^{(k)}] \\
    &= \tilde{x}^{(k)\top} E[ (X^{(c)\top} X^{(c)})^{-1} ( X^{(1)\top} X^{(1)} \beta^{(1)} + X^{(2)\top} X^{(2)} \beta^{(2)} ) | \tilde{x}^{(k)} ] \\
    &= \tilde{x}^{(k)\top} \beta^{(k)} + \tilde{x}^{(k)\top} E[ (X^{(c)\top} X^{(c)})^{-1} X^{(3-k)\top} X^{(3-k)} (\beta^{(3-k)}-\beta^{(k)})  | \tilde{x}^{(k)} ]\\
    &= \tilde{x}^{(k)\top} \beta^{(k)} + \tilde{x}^{(k)\top} E[ Z_{3-k} (\beta^{(3-k)}-\beta^{(k)})  | \tilde{x}^{(k)} ],
\end{align*}
hence
\begin{align} \label{aeq1}
    E[ \tilde{x}^{(k)\top} \beta^{(k)} - E[\tilde{x}^{(k)\top} \hat{\beta}^{(c)} | \tilde{x}^{(k)}])^2] 
    &= (\beta^{(1)}-\beta^{(2)})^\top E[ Z_{3-k} ]^\top E[\tilde{x}^{(k)} \tilde{x}^{(k)\top}] E[Z_{3-k}]  (\beta^{(1)}-\beta^{(2)}).
\end{align}
Also, we have
\begin{align*}
    E & \left[ (\tilde{x}^{(k)\top} \beta^{(k)} - E[\tilde{x}^{(k)\top} \hat{\beta}^{(c)} | \tilde{x}^{(k)}]) ( E[\tilde{x}^{(k)\top} \hat{\beta}^{(c)} | \tilde{x}^{(k)}]- \tilde{x}^{(k)\top} \hat{\beta}^{(c)}) \right]  \\
    &=	E \left[ E \left[ (\tilde{x}^{(k)\top} \beta^{(k)} - E[\tilde{x}^{(k)\top} \hat{\beta}^{(c)} | \tilde{x}^{(k)}]) ( E[\tilde{x}^{(k)\top} \hat{\beta}^{(c)} | \tilde{x}^{(k)}] - \tilde{x}^{(k)\top} \hat{\beta}^{(c)}) \big| \tilde{x}^{(k)} \right] \right]  \\
    &= 0.
\end{align*}
By Lemma \ref{alemma2} and \eqref{aeq1}, it follows that 
\begin{align*}
    OSE(\hat{\beta}^{(c)}, P^{(k)}_{\mathcal{X} \times \mathcal{Y}}) 
    &= \sigma^2 
        + \sigma^2 tr( W_k E[(X^{(c)\top} X^{(c)})^{-1}] ) 
        + (\beta^{(1)} - \beta^{(2)})^\top E[Z_{3-k}^\top W_k Z_{3-k}] (\beta^{(1)} - \beta^{(2)}).
\end{align*}
\end{proof}

\begin{lemma} \label{alemma2}
Under Assumption \ref{assumption1},
\begin{align*}
    E[( E[\tilde{x}^{(k)\top} \hat{\beta}^{(c)} | \tilde{x}^{(k)}] - \tilde{x}^{(k)\top} \hat{\beta}^{(c)})^2] 
    &= \sigma^2 tr\left( W_k E[ (X^{(c)\top} X^{(c)})^{-1} ]\right) \\
    & \quad + (\beta^{(k)} - \beta^{(3-k)})^\top E[ Z_{3-k}^\top W_k Z_{3-k}] (\beta^{(k)} - \beta^{(3-k)})  \\
    &- (\beta^{(k)} - \beta^{(3-k)})^\top E[Z_{3-k}]^\top W_k E[ Z_{3-k}] (\beta^{(k)} - \beta^{(3-k)}).
\end{align*}
\end{lemma}
\begin{proof}
We first derive
\begin{align*}
    E[( E[\tilde{x}^{(k)\top} \hat{\beta}^{(c)} | \tilde{x}^{(k)}] - \tilde{x}^{(k)\top} \hat{\beta}^{(c)})^2]
    &= E[ E[( E[\tilde{x}^{(k)\top} \hat{\beta}^{(c)} | \tilde{x}^{(k)}] - \tilde{x}^{(k)\top} \hat{\beta}^{(c)})^2 | \tilde{x}^{(k)}]] \\
    &= E[Var(\tilde{x}^{(k)\top} \hat{\beta}^{(c)} | \tilde{x}^{(k)})],
\end{align*}
and
\begin{align*}
    Var(\tilde{x}^{(k)\top} \hat{\beta}^{(c)} | \tilde{x}^{(k)}) 
    &= E[ Var(\tilde{x}^{(k)\top} \hat{\beta}^{(c)} | X^{(c)}, \tilde{x}^{(k)}) | \tilde{x}^{(k)}] 
    + Var(E[ \tilde{x}^{(k)\top} \hat{\beta}^{(c)} | X^{(c)}, \tilde{x}^{(k)\top}] | \tilde{x}^{(k)}) \\
    &= \sigma^2 \tilde{x}^{(k)\top} E[ (X^{(c)\top} X^{(c)})^{-1} ] \tilde{x}^{(k)\top} \\
    &\quad + Var\left(\tilde{x}^{(k)\top} \left( \beta^{(k)} + (X^{(c)\top} X^{(c)})^{-1} X^{(3-k)\top} X^{(3-k)} (\beta^{(3-k)}- \beta^{(k)} ) \right) | \tilde{x}^{(k)}\right) \\
    &= \sigma^2 \tilde{x}^{(k)\top} E[ (X^{(c)\top} X^{(c)})^{-1} ] \tilde{x}^{(k)\top} + \tilde{x}^{(k)\top} Var \left(Z_{3-k} (\beta^{(k)}- \beta^{(3-k)} ) \right) \tilde{x}^{(k)},
\end{align*}
hence 
\begin{align*}
    E[( E[\tilde{x}^{(k)\top} \hat{\beta}^{(c)} | \tilde{x}^{(k)}] - \tilde{x}^{(k)\top} \hat{\beta}^{(c)})^2]
    &= \sigma^2 tr\left( W_{k} E[ (X^{(c)\top} X^{(c)})^{-1} ]\right) + tr\left( W_{k} Var \left(Z_{3-k} (\beta^{(k)}- \beta^{(3-k)} ) \right)  \right) .
\end{align*}
To further simplify $Var \left(Z_{3-k} (\beta^{(k)}- \beta^{(3-k)} ) \right)$,
\begin{align*}
    Var \left(Z_{3-k} (\beta^{(k)}- \beta^{(3-k)} ) \right)  
    &= E[Z_{3-k} (\beta^{(k)} - \beta^{(3-k)})(\beta^{(k)} - \beta^{(3-k)})^\top Z_{3-k}^\top] \\
    &\quad - E[Z_{3-k}] (\beta^{(k)} - \beta^{(3-k)}) (\beta^{(k)} - \beta^{(3-k)})^\top E[Z_{3-k}]^\top.
\end{align*}
Since
\begin{align*}
    E[ Z_{3-k} (\beta^{(k)} - \beta^{(3-k)})(\beta^{(k)} - \beta^{(3-k)})^\top Z_{3-k}^\top]_{i,j} 
    &=E[ (Z_{3-k})_{i,:} (\beta^{(k)} - \beta^{(3-k)})(\beta^{(k)} - \beta^{(3-k)})^\top (Z_{3-k})_{j,:}^\top] \\
    &= (\beta^{(k)} - \beta^{(3-k)})^\top E[ (Z_{3-k})_{j,:}^\top (Z_{3-k})_{i,:}] (\beta^{(k)} - \beta^{(3-k)}),
\end{align*}
where $(Z_{3-k})_{i, :}$ refers to the $i$th row of $Z_{3-k}$, it follows that
\begin{align*}
    tr \left(W_k Var\left( Z_{3-k} (\beta^{(k)} - \beta^{(3-k)}) \right) \right) 
    &= tr\left( W_k \left[ (\beta^{(k)} - \beta^{(3-k)})^\top E[ (Z_{3-k})_{j,:}^\top (Z_{3-k})_{i,:}] (\beta^{(k)} - \beta^{(3-k)})  \right]_{i,j = 1, \cdots, p} \right)\\
    & \quad - (\beta^{(k)} - \beta^{(3-k)})^\top E[Z_{3-k}]^\top W_k E[ Z_{3-k}] (\beta^{(k)} - \beta^{(3-k)}). 
\end{align*}
Now, because we have the following relations
\begin{align*}
    tr & \left( W_k \left[ (\beta^{(k)} - \beta^{(3-k)})^\top E[ (Z_{3-k})_{j,:}^\top (Z_{3-k})_{i,:}] (\beta^{(k)} - \beta^{(3-k)}) \right]_{i,j = 1, \cdots, p} \right) \\
    &= \sum_{j=1}^p (W_k)_{1,j} (\beta^{(k)} - \beta^{(3-k)})^\top E[ (Z_{3-k})_{j,:}^\top (Z_{3-k})_{1,:}] (\beta^{(k)} - \beta^{(3-k)}) \\
    & \quad + \cdots + \sum_{j=1}^p (W_k)_{p,j} (\beta^{(k)} - \beta^{(3-k)})^\top E[ (Z_{3-k})_{j,:}^\top (Z_{3-k})_{p,:}] (\beta^{(k)} - \beta^{(3-k)}) \\
    &= \sum_{i=1}^p \sum_{j=1}^p (W_k)_{i,j} (\beta^{(k)} - \beta^{(3-k)})^\top E[ (Z_{3-k})_{j,:}^\top (Z_{3-k})_{i,:}] (\beta^{(k)} - \beta^{(3-k)}) \\
    &= (\beta^{(k)} - \beta^{(3-k)})^\top \sum_{i=1}^p \sum_{j=1}^p (W_k)_{i,j}  E[ (Z_{3-k})_{j,:}^\top (Z_{3-k})_{i,:}] (\beta^{(k)} - \beta^{(3-k)})\\
    &= (\beta^{(k)} - \beta^{(3-k)})^\top E[ Z_{3-k}^\top W_k Z_{3-k} (\beta^{(k)} - \beta^{(3-k)}),
\end{align*}
we get
\begin{align*}
    E[( E[\tilde{x}^{(k)\top} \hat{\beta}^{(c)} | \tilde{x}^{(k)}] - \tilde{x}^{(k)\top} \hat{\beta}^{(c)})^2]
    &= \sigma^2 tr\left( W_k E[ (X^{(c)\top} X^{(c)})^{-1} ]\right) \\
    & \quad + (\beta^{(k)} - \beta^{(3-k)})^\top E[ Z_{3-k}^\top W_k Z_{3-k}] (\beta^{(k)} - \beta^{(3-k)})  \\
    &- (\beta^{(k)} - \beta^{(3-k)})^\top E[Z_{3-k}]^\top W_k E[ Z_{3-k}] (\beta^{(k)} - \beta^{(3-k)}).
\end{align*}
\end{proof}

Now that the exact error is identified, the equivalent condition in Theorem \ref{thm1} for error reduction can be proved.

\begin{reptheorem}{thm1}\label{thm1rep}
Under Assumption \ref{assumption1}, we have
\begin{align*}
    \sum_{k=1}^2 \text{OSE}(\hat{\beta}^{(k)}, P^{(k)}_{\mathcal{X} \times \mathcal{Y}}) > \sum_{k=1}^2 \text{OSE}(\hat{\beta}^{(c)}, P^{(k)}_{\mathcal{X} \times \mathcal{Y}}) 
\end{align*} 
if and only if 
\begin{align*}
    h(\sigma^2) > g(\beta^{(1)}, \beta^{(2)}),
\end{align*} 
where 
\begin{align*}
    h(x) &= A_0 x, \\
    g(y,z) &= \lVert y-z \rVert_{B_0}^2, \\
    A_0 &= tr\left( W_1 E\left[ (X^{(1)\top}X^{(1)})^{-1} - (X^{(c)\top} X^{(c)})^{-1} \right] + W_2 E\left[ (X^{(2)\top}X^{(2)})^{-1} - (X^{(c)\top} X^{(c)})^{-1} \right]\right), \text{ and} \\
    B_0 &= E[ Z_1^\top W_2 Z_1] + E[ Z_2^\top W_1 Z_2] .
\end{align*}
\end{reptheorem}
\begin{proof}
By Lemma \ref{alemma1},
\begin{align*}
    \sum_{k=1}^2 & \text{OSE}(\hat{\beta}^{(k)}, P^{(k)}_{\mathcal{X} \times \mathcal{Y}}) > \sum_{k=1}^2 \text{OSE}(\hat{\beta}^{(c)}, P^{(k)}_{\mathcal{X} \times \mathcal{Y}}) \\
    & \iff 2\sigma^2 + \sigma^2 tr\left( W_1 E\left[ (X^{(1)\top}X^{(1)})^{-1} \right] + W_2 E\left[ (X^{(2)\top}X^{(2)})^{-1}\right]\right)   > \sigma^2 tr\left( (W_1+W_2) E[ (X^{(c)\top} X^{(c)})^{-1} ]\right) \\
    & \quad + (\beta^{(1)} - \beta^{(2)})^\top \left( E[ Z_1^\top W_2 Z_1] + E[ Z_2^\top W_1 Z_2]\right) (\beta^{(1)} - \beta^{(2)} ) + 2\sigma^2	\\
    & \iff \sigma^2  tr\left( W_1 E\left[ (X^{(1)\top}X^{(1)})^{-1} - (X^{(c)\top} X^{(c)})^{-1} \right] + W_2 E\left[ (X^{(2)\top}X^{(2)})^{-1} - (X^{(c)\top} X^{(c)})^{-1} \right]\right)  \\
    & \quad > (\beta^{(1)} - \beta^{(2)})^\top \left( E[ Z_1^\top W_2 Z_1] + E[ Z_2^\top W_1 Z_2]\right) (\beta^{(1)} - \beta^{(2)} ) .
\end{align*} 
\end{proof}

\subsubsection{Replacing Parameters with Estimators} \label{phi_psi_section}

Now, we prove Lemma \ref{lemma1}, which replaces $\beta^{(k)}$ and $\sigma^2$ with estimators by imposing high probability bounds. We first state concentration and anti-concentration inequalities from existing literature, and then develop high probability bounds on both sides of the inequality in Theorem \ref{thm1}. 

\begin{theorem} \label{hw} \citep{hsu2011}
If $K \in \mathbb{R}^{m \times n}, \ L = K^\intercal K, \ \mu \in \mathbb{R}^n, \ \sigma \geq 0,$ and $ E[e^{\alpha^\intercal (x - \mu)}] \leq e^{\frac{1}{2} \sigma^2 \lVert \alpha \rVert^2 }, ~  \forall \alpha \in \mathbb{R}^n$, then for $\forall t >0$
\begin{align*}
    \mathbb{P} \left( \lVert Kx \rVert^2 > \sigma^2 \left\{ tr(L)  + 2 \sqrt{tr(L^2) t} + 2 \lVert L \rVert t \right\} + \lVert K \mu \rVert^2 \sqrt{ 1 + 4 \sqrt{ \frac{\lVert L^2 \rVert}{tr( L^2)} t } + \frac{4 \lVert L \rVert^2}{tr(L^2)}t } \right) \leq e^{-t} .
\end{align*}
\end{theorem}

\begin{theorem} \label{cw} \citep{lovett2010}
If $V \sim N_d(0,I)$, then
\begin{align*}
    \mathbb{P}(\lVert V \rVert^2 \leq \epsilon) \leq \sqrt{2e\frac{\epsilon}{d}} .
\end{align*}
\end{theorem}

Next, we present lemmas needed for the application of our setting.
\begin{lemma} \label{cwlemma}
If $Z \sim N_d (\mu, I_d)$, then 
\begin{align*}
    \mathbb{P}(\lVert Z \rVert \leq t) \leq \sqrt{\frac{2e}{d}} (\lVert \mu \rVert +t) + 2e^{-\frac{(\lVert \mu \rVert -t)^2}{2d}}-1
\end{align*} 
for $t < \lVert \mu \rVert$.
\end{lemma}
\begin{proof} From Theorem \ref{cw}, we can show that
\begin{align*}
    \mathbb{P}(\lVert Z \rVert \leq t) 
    &\leq \mathbb{P}(\lVert Z -\mu \rVert \leq \lVert \mu \rVert + t) - \mathbb{P}(\lVert Z -\mu \rVert \leq \lVert \mu \rVert - t) \\
    &\leq \sqrt{\frac{2e}{d}} (\lVert \mu \rVert +t) + 2e^{-\frac{(\lVert \mu \rVert -t)^2}{2d}}-1
\end{align*} 
\end{proof}

Based on these inequalities, a lower bound for $f(\sigma^2)$ can be obtained as follows.

\begin{lemma} \label{alemma3}
For $D^\top D = \{(X^{(1)\top}X^{(1)})^{-1} + (X^{(2)\top}X^{(2)})^{-1}\}^{-1}$,
    \begin{align*}
        h(\sigma^2) \geq \phi_\delta (\hat{\beta}^{(1)}, \hat{\beta}^{(2)}, \hat{\sigma}^{(c) 2})
    \end{align*}
with probability at least $1-2\delta$, where 
\begin{align*}
    \phi_\delta(u,v,w) &= \tilde{A}_1(\delta)w + \tilde{A}_2(\delta) \lVert D(u-v) \rVert^2, \\
    \tilde{A}_1(\delta) &= \frac{n_0 A_0}{\left( n_0 + 2 \sqrt{n_0\log{\frac{1}{\delta}}} + 2 \log{\frac{1}{\delta}} \right) \left( 3 + 4\sqrt{\frac{1}{n_0}\log{\frac{1}{\delta}}} \right) }, \\
    \tilde{A}_2(\delta) &= \frac{- 2 A_0 \left( 1 + 2\sqrt{\frac{1}{n_0}\log{\frac{1}{\delta}}} \right)}{\left( n_0 + 2 \sqrt{n_0\log{\frac{1}{\delta}}} + 2 \log{\frac{1}{\delta}} \right) \left( 3 + 4\sqrt{\frac{1}{n_0}\log{\frac{1}{\delta}}} \right) } ,
\end{align*}  
and $n_0 = n_1+n_2-p$.
\end{lemma}
\begin{proof}
Let us introduce additional notations for more efficient computation. Let $X^{(0)} =\begin{bmatrix} X^{(1)} & 0 \\ 0 & X^{(2)} \\ \end{bmatrix}$, 
$y^{(0)} = \begin{bmatrix} y^{(1)}\\ y^{(2)} \\ \end{bmatrix}$,
$\epsilon^{(0)} = \begin{bmatrix} \epsilon^{(1)}\\ \epsilon^{(2)} \\ \end{bmatrix}$,
and $\beta^{(0)} = \begin{bmatrix} \beta^{(1)}\\ \beta^{(2)} \\ \end{bmatrix}$. Then the estimator for $\sigma^2$ on the combined dataset $D_1 \cup D_2$ is $\hat{\sigma}^{(c)2} = \frac{1}{n_1+n_2-p} (\lVert y^{(1)} - X^{(1)} \hat{\beta}^{(c)} \rVert^2 + \lVert y^{(2)} - X^{(2)} \hat{\beta}^{(c)} \rVert^2)$. 

Let us further denote $C = \begin{bmatrix} I_p & -I_p \end{bmatrix}$ and $A = X^{(0)} ( X^{(0)\top} X^{(0)})^{-1} C^\top$, and let $M = A (A^\top A)^{-1} A^\top$ and $H_{X^{(0)}} = X^{(0)} ( X^{(0)\top} X^{(0)})^{-1} X^{(0)}$. Then $I - H_{X^{(0)}} + M$ is symmetric and idempotent, and $\hat{\sigma}^{(c)2} = \frac{1}{n_1+n_2-p} y^{(0) \top} (I - H_{X^{(0)}} + M) y^{(0)}$. Since $tr(I- H_{X^{(0)}}+M) = n_1+n_2-p$, $\lVert I- H_{X^{(0)}}+M \rVert = \lambda_{max}(I- H_{X^{(0)}}+M) = 1 $, and 
\begin{align*}
    \lVert (I- H_{X^{(0)}}+M)X^{(0)} \beta^{(0)} \rVert^2 
    &= \beta^{(0) \top} X^{(0)\top} ( I - H_{X^{(0)}} + M) X^{(0)} \beta^{(0)} \\
    &= \beta^{(0) \top} C^\top \{C (X^{(0)\top} X^{(0)})^{-1} C^\top\}^{-1} C \beta^{(0)},
\end{align*}
Theorem \ref{hw} gives 
\begin{align*}
    (n_1+n_2-p)\hat{\sigma}^{(c) 2} &\leq \sigma^2 \left( n_1 + n_2 - p + 2 \sqrt{n_1 + n_2 -p}\sqrt{\log{\frac{1}{\delta}}} + 2 \log{\frac{1}{\delta}} \right) \\
    & \quad + \left\{ 1 + 2\sqrt{\frac{\log{\frac{1}{\delta}}}{n_1 + n_2 - p}} \right\} \lVert D(\beta^{(1)} - \beta^{(2)}) \rVert^2
    \stepcounter{equation}\tag{\theequation}\label{aeq2}
\end{align*} with probability at least $1-\delta$, where $\lambda_{max}(\cdot)$ denotes the maximum eigenvalue of a matrix. 

As $D(\hat{\beta}^{(1)}-\hat{\beta}^{(2)}) = DC \hat{\beta}^{(0)} = DA^\intercal y^{(0)}$, Theorem \ref{hw} leads to 
\begin{align*}
    \mathbb{P}(\lVert D( \hat{\beta}^{(1)} - \hat{\beta}^{(2)} ) - D( \beta^{(1)} - \beta^{(2)}) \rVert^2 > \sigma^2(tr(M) + 2\sqrt{tr(M)t} + 2 \lVert M \rVert t)  ) \leq e^{-t},
\end{align*} or equivalently,
\begin{align*}
    \lVert D( \hat{\beta}^{(1)} - \hat{\beta}^{(2)} ) - D( \beta^{(1)} - \beta^{(2)} ) \rVert \leq \sqrt{ \sigma^2 \left( n_1+n_2-p + 2\sqrt{n_1+n_2-p}\sqrt{\log{\frac{1}{\delta}}} + 2\log{\frac{1}{\delta}} \right)}
    \stepcounter{equation}\tag{\theequation}\label{aeq4}
\end{align*} 
with probability at least $1-\delta$. Hence
\begin{align*}
    \lVert D( \beta^{(1)} - \beta^{(2)} ) \rVert^2 
    & \leq 2 \lVert D( \hat{\beta}^{(1)} - \hat{\beta}^{(2)} ) - D( \beta^{(1)} - \beta^{(2)} ) \rVert^2 + 2 \lVert D( \hat{\beta}^{(1)} - \hat{\beta}^{(2)} ) \rVert^2 \\
    & \leq 2 \sigma^2 \left(n_1+n_2-p + 2\sqrt{n_1+n_2-p}\sqrt{\log{\frac{1}{\delta}}} + 2\log{\frac{1}{\delta}}\right) + 2 \lVert D( \hat{\beta}^{(1)} - \hat{\beta}^{(2)} ) \rVert^2
    \stepcounter{equation}\tag{\theequation}\label{aeq3}
\end{align*} 
with probability at least $1-\delta$. Combining \eqref{aeq2} and \eqref{aeq3} gives
\begin{align*}
    (n_1+n_2-p) \hat{\sigma}^{(c) 2} &\leq \sigma^2 \left( n_1 + n_2 - p + 2 \sqrt{n_1 + n_2 -p}\sqrt{\log{\frac{1}{\delta}}} + 2 \log{\frac{1}{\delta}} \right) \left( 3 + 4\sqrt{\frac{\log{\frac{1}{\delta}}}{n_1 + n_2 - p}} \right) \\
    & \quad + 2 \left( 1 + 2\sqrt{\frac{\log{\frac{1}{\delta}}}{n_1 + n_2 - p}} \right) \lVert D( \hat{\beta}^{(1)} - \hat{\beta}^{(2)} ) \rVert^2 ,
\end{align*} or equivalently,
\begin{align*}
    h(\sigma^2) \geq \tilde{A}_1 (\delta) \hat{\sigma}^{(c)2} + \tilde{A}_2 (\delta) \lVert D(\hat{\beta}^{(1)} - \hat{\beta}^{(2)}) \rVert^2 .
\end{align*}
with probability at least $1-2\delta$.
\end{proof}

Next, $g(\beta^{(1)}, \beta^{(2)})$ can also be bounded in a similar way, and we show it in the following lemma.

\begin{lemma} \label{alemma4}
For $D^\top D = \{(X^{(1)\top}X^{(1)})^{-1} + (X^{(2)\top}X^{(2)})^{-2}\}^{-1}$,
    \begin{align*}
        g(\beta^{(1)}, \beta^{(2)}) \leq \psi_\delta (\hat{\beta}^{(1)}, \hat{\beta}^{(2)}, \hat{\sigma}^{(c) 2})
    \end{align*}
with probability at least $1-2\delta$, where 
\begin{align*}
    \psi_\delta(u,v,w) &= \left\{ \sqrt{g(u,v)} + T_\delta (w, \lVert D(u-v) \rVert) \right\}^2 
\end{align*}
and $T_\delta (x,y) = O(\max \{ x,y \} )$.
\end{lemma}
\begin{proof}
Let $I- H_{X^{(0)}}+M = P \Lambda P^\intercal$. Since $I- H_{X^{(0)}}+M$ is symmetric and idempotent, its eigenvalues are 0 or 1. As $tr(I- H_{X^{(0)}}+M) = n_1 + n_2 -p$, assume the first $n_1-n_2+p$ diagonal entries of $\Lambda$ be 1 and the rest be 0 without loss of generality. Then
\begin{align*}
    P^\top (I- H_{X^{(0)}}+M) y^{(0)} & \sim N ( P^\top M X^{(0)} \beta^{(0)}, \sigma^2 \Lambda) \\
    & = N_{n_1+n_2} \left( P^\intercal M X^{(0)} \beta^{(0)}, \sigma^2 \begin{bmatrix}
        I_{n_1+n_2-p} & 0 \\
        0 & 0
    \end{bmatrix}\right).
\end{align*}
where $I_n$ denotes the $n \times n$ identity matrix.

Denote $F = \begin{bmatrix} I_{n_1+n_2-p} & 0 \end{bmatrix}$ and $G = \begin{bmatrix} 0 & I_p \end{bmatrix}$. For a constant $c$,
\begin{align*}
    \mathbb{P}( (n_1 +n_2-p) \hat{\sigma}^{(c)2} \leq c) 
    &= \mathbb{P}( \lVert P^\top (I- H_{X^{(0)}}+M) y^{(0)} \rVert^2 \leq c) \\
    &= \mathbb{P}( \lVert F P^\top (I- H_{X^{(0)}}+M) y^{(0)} \rVert^2 + \lVert G P^\top (I- H_{X^{(0)}}+M) y^{(0)} \rVert^2 \leq c) 
\end{align*}
Let $Z = F P^\intercal (I- H_{X^{(0)}}+M) y^{(0)}$. Then $Z \sim N ( FP^\intercal M X^{(0)} \beta^{(0)} , \sigma^2 I)$, and 
\begin{align*}
    \lVert FP^\top M X^{(0)} \beta^{(0)} \rVert^2
    &= \beta^{(0)\top} X^{(0)\top} M (I - H_{X^{(0)}} + M) M X^{(0)} \beta^{(0)} \\
    &= \beta^{(0)\top} C^\top (C (X^{(0)\top}X^{(0)})^{-1} C^\top )^{-1} C \beta^{(0)}  \\
    &= \lVert D(\beta^{(1)}-\beta^{(2)}) \rVert^2.
\end{align*}
Also, $GP^\top (I- H_{X^{(0)}}+M) y^{(0)} \sim N ( GP^\top M X^{(0)} \beta^{(0)}, 0)$, and 
\begin{align*}
    \lVert G P^\top M X^{(0)} \beta^{(0)}\rVert^2 
    &= \beta^{(0)\top} X^{(0)\top} M P (I-\Lambda) P^\top M X^{(0)} \beta^{(0)} \\
    &= \beta^{(0)\top} X^{(0)\top}M (I-I+H_{X^{(0)}}-M) M X^{(0)} \beta^{(0)} \\
    &=0 .
\end{align*}
Combining with Lemma \ref{cwlemma}, it follows that
\begin{align*}
    \mathbb{P}( (n_1 +n_2-p) \hat{\sigma}^{(c)2} \leq c) 
    &\leq \frac{1}{\sigma}\sqrt{\frac{2e}{n_1+n_2-p}}( \lVert D(\beta^{(1)}-\beta^{(2)}) \rVert + \sqrt{c}) + 2e^{-\frac{(\lVert D(\beta^{(1)}-\beta^{(2)})\rVert - \sqrt{c})^2}{2(n_1+n_2-p)}}-1\\
    &\leq \frac{2}{\sigma}\sqrt{\frac{2e}{n_1+n_2-p}} \lVert D(\beta^{(1)}-\beta^{(2)}) \rVert + 2e^{-\frac{(\lVert D(\beta^{(1)}-\beta^{(2)})\rVert - \sqrt{c})^2}{2(n_1+n_2-p)}}-1\\
\end{align*} for $c \leq \lVert D(\beta^{(1)}-\beta^{(2)})\rVert^2$. Hence
\begin{align*}
    (n_1  + n_2-p) \hat{\sigma}^{(c)2} 
    &\geq \bigg\{ \lVert D(\beta^{(1)}-\beta^{(2)})\rVert + \sigma \sqrt{2(n_1+n_2-p) \log{\frac{2}{1+\delta - 2 \lVert D(\beta^{(1)}-\beta^{(2)})\rVert \sqrt{\frac{2e}{n_1+n_2-p}}\frac{1}{\sigma}}}}\bigg\}^2 \\
    & \geq \bigg\{ \lVert D(\beta^{(1)}-\beta^{(2)})\rVert + \sigma \sqrt{2(n_1+n_2-p)} \bigg( \sqrt{\log{\frac{4}{1+\delta}}} + \frac{1}{2\sqrt{\log{\frac{4}{1+\delta}}}} -\frac{1}{8\lVert D(\beta^{(1)}-\beta^{(2)})\rVert} \\
    & \quad \times \sqrt{\frac{n_1+n_2-p}{2e}} \frac{1+\delta}{\sqrt{\log{\frac{4}{1+\delta}}}} \sigma \bigg) \bigg\}^2
\end{align*}
with probability at least $1-\delta$ as the function $\xi: s \mapsto \sqrt{\log{\frac{1}{a-\frac{b}{s}}}}$ satisfies
\begin{align*}
    \xi(s) 
    \geq \xi\left(\frac{2b}{a}\right) + \xi'\left(\frac{2b}{a}\right)\left(s-\frac{2b}{a}\right) 
    = -\frac{a}{4b\sqrt{\log{\frac{4}{a}}}}s + \sqrt{\log{\frac{4}{a}}} + \frac{1}{2\sqrt{\log{\frac{4}{a}}}} .
\end{align*} 
Rearranging the terms gives
\begin{align*}
    \lVert D(\beta^{(1)}-\beta^{(2)})\rVert \sqrt{(n_1 + n_2-p) \hat{\sigma}^{(c)2} } 
    & \geq -\frac{n_1+n_2-p}{8}\frac{1+\delta}{\sqrt{e\log{\frac{4}{1+\delta}}}} \sigma^2 + +\lVert D(\beta^{(1)}-\beta^{(2)})\rVert ^2 \\
    & + \lVert D(\beta^{(1)}-\beta^{(2)})\rVert \sqrt{2(n_1 + n_2-p)} \left(\sqrt{\log{\frac{4}{1+\delta}}} + \frac{1}{2\sqrt{\log{\frac{4}{1+\delta}}}}\right) \sigma  
\stepcounter{equation}\tag{\theequation}\label{aeq5}
\end{align*} with probability at least $1-\delta$. 

Meanwhile, from \eqref{aeq4},
\begin{align*}
    \lVert D( \beta^{(1)} - \beta^{(2)} ) \rVert 
    &\leq \lVert D( \hat{\beta}^{(1)} - \hat{\beta}^{(2)} ) - D( \beta^{(1)} - \beta^{(2)} ) \rVert + \lVert D( \hat{\beta}^{(1)} - \hat{\beta}^{(2)} )  \rVert \\
    &\leq \sqrt{ \sigma^2 (n_1+n_2-p + 2\sqrt{n_1+n_2-p}\sqrt{\log{\frac{1}{\delta}}} + 2\log{\frac{1}{\delta}})} +\lVert D( \hat{\beta}^{(1)} - \hat{\beta}^{(2)} )  \rVert
\stepcounter{equation}\tag{\theequation}\label{aeq6}
\end{align*} and 
\begin{align*}
    \lVert D( \beta^{(1)} - \beta^{(2)} ) \rVert 
    &\geq \lVert D( \hat{\beta}^{(1)} - \hat{\beta}^{(2)} ) \rVert -\lVert D( \hat{\beta}^{(1)} - \hat{\beta}^{(2)} ) - D( \beta^{(1)} - \beta^{(2)} ) \rVert \\
    &\geq \lVert D( \hat{\beta}^{(1)} - \hat{\beta}^{(2)} ) \rVert - \sqrt{ \sigma^2 (n_1+n_2-p + 2\sqrt{n_1+n_2-p}\sqrt{\log{\frac{1}{\delta}}} + 2\log{\frac{1}{\delta}})}
\stepcounter{equation}\tag{\theequation}\label{aeq7}
\end{align*} with probability at least $1-\delta$.

Combining \eqref{aeq5}, \eqref{aeq6}, and \eqref{aeq7} gives
\begin{align*}
    \sqrt{(n_1 + n_2-p) \hat{\sigma}^{(c)2}} & \left( \sqrt{ \sigma^2 (n_1+n_2-p + 2\sqrt{n_1+n_2-p}\sqrt{\log{\frac{1}{\delta}}} + 2\log{\frac{1}{\delta}})} +\lVert D( \hat{\beta}^{(1)} - \hat{\beta}^{(2)} ) \rVert \right) \\
    & \geq \lVert D(\beta^{(1)}-\beta^{(2)})\rVert \sqrt{(n_1 + n_2-p)  \hat{\sigma}^{(c)2}} \\
    & \geq -\frac{n_1+n_2-p}{8}\frac{1+\delta}{\sqrt{e\log{\frac{4}{1+\delta}}}} \sigma^2 + \lVert D(\beta^{(1)}-\beta^{(2)})\rVert ^2 \\
    & \quad + \lVert D(\beta^{(1)}-\beta^{(2)})\rVert \sqrt{2(n_1 + n_2-p)}  \left(\sqrt{\log{\frac{4}{1+\delta}}} + \frac{1}{2\sqrt{\log{\frac{4}{1+\delta}}}}\right) \sigma  \\
    & \geq -\frac{n_1+n_2-p}{8}\frac{1+\delta}{\sqrt{e\log{\frac{4}{1+\delta}}}} \sigma^2 + \left(\sqrt{\log{\frac{4}{1+\delta}}} + \frac{1}{2\sqrt{\log{\frac{4}{1+\delta}}}}\right) \sqrt{2(n_1 + n_2-p)} \sigma\\
    &\quad \times \bigg( \lVert D( \hat{\beta}^{(1)} - \hat{\beta}^{(2)} )  \rVert - \sigma \sqrt{ (n_1+n_2-p + 2\sqrt{n_1+n_2-p}\sqrt{\log{\frac{1}{\delta}}} + 2\log{\frac{1}{\delta}})} \bigg) \\ 
    &\quad +\bigg( \lVert D( \hat{\beta}^{(1)} - \hat{\beta}^{(2)} ) \rVert - \sqrt{ \sigma^2 (n_1+n_2-p + 2\sqrt{n_1+n_2-p}\sqrt{\log{\frac{1}{\delta}}} + 2\log{\frac{1}{\delta}})} \bigg)^2
\end{align*} 
with probability at least $1-2\delta$. Solving $\sigma$ leads to
\begin{align*}
    \sigma \leq c_0 := \frac{c_1-c_4+\sqrt{(c_1-c_4)^2-4c_3(c_5-c_2)}}{2c_3}
\end{align*} 
with probability at least $1-2\delta$, and the parameters being
\begin{align*}
    c_1 &=  \sqrt{(n_1 + n_2-p)  \hat{\sigma}^{(c)2}} 
        \sqrt{ n_1+n_2-p + 2\sqrt{n_1+n_2-p}\sqrt{\log{\frac{1}{\delta}}}+ 2\log{\frac{1}{\delta}}}, \\
    c_2 &= \sqrt{(n_1 + n_2-p)  \hat{\sigma}^{(c)2}} \lVert D( \hat{\beta}^{(1)} - \hat{\beta}^{(2)} )  \rVert, \\
    c_3 &= -\frac{n_1+n_2-p}{8}\frac{1+\delta}{\sqrt{e\log{\frac{4}{1+\delta}}}}  - \sqrt{2(n_1+n_2-p)}
        \left(\sqrt{\log{\frac{4}{1+\delta}}} + \frac{1}{2\sqrt{\log{\frac{4}{1+\delta}}}}\right) \\
    & \quad \times \sqrt{ n_1+n_2-p + 2\sqrt{n_1+n_2-p}\sqrt{\log{\frac{1}{\delta}}}+ 2\log{\frac{1}{\delta}}}  + n_1+n_2-p + 2\sqrt{n_1+n_2-p}\sqrt{\log{\frac{1}{\delta}}}+ 2\log{\frac{1}{\delta}} , \\
    c_4 &= \lVert D( \hat{\beta}^{(1)} - \hat{\beta}^{(2)} ) \rVert 
        \Bigg\{ \sqrt{2(n_1+n_2-p)} \left(\sqrt{\log{\frac{4}{1+\delta}}} + \frac{1}{2\sqrt{\log{\frac{4}{1+\delta}}}}\right) \\
    & \quad -2 \sqrt{ n_1+n_2-p + 2\sqrt{n_1+n_2-p}\sqrt{\log{\frac{1}{\delta}}}+ 2\log{\frac{1}{\delta}}} \Bigg\}, \text{ and}\\
    c_5 &= \lVert D( \hat{\beta}^{(1)} - \hat{\beta}^{(2)} ) \rVert^2 .
\end{align*}

Now, for $B_0 = B^\top B$, applying Theorem \ref{hw} leads to 
\begin{align*}
    \lVert B (\beta^{(1)}-\beta^{(2)}) - B(\hat{\beta}^{(1)}-\hat{\beta}^{(2)}) \rVert^2 
    &=\lVert BC (X^{(0) \top} X^{(0)})^{-1} X^{(0)\top} (X^{(0)} \beta^{(0)} - y^{(0)}) \rVert^2 \\
    &\leq \sigma^2 \left\{ tr(\Sigma) + 2 \sqrt{tr(\Sigma^2) \log{\frac{1}{\delta}}} + 2 \lVert \Sigma \rVert \log{\frac{1}{\delta}} \right\}
\end{align*}
with probability at least $1-\delta$, and $\Sigma = X^{(0)} (X^{(0)\top} X^{(0)})^{-1} C^\top B^\top BC (X^{(0)\top} X^{(0)})^{-1} X^{(0)\top}$. Therefore, 
\begin{align*}
    g(\beta^{(1)}, \beta^{(2)})
    &\leq \left\{ \lVert B ( \beta^{(1)} - \beta^{(2)} ) -  B( \hat{\beta}^{(1)} - \hat{\beta}^{(2)} ) \rVert + \lVert B( \hat{\beta}^{(1)} - \hat{\beta}^{(2)} ) \rVert \right\}^2 \\
    &\leq \left\{ \sigma \sqrt{ tr(\Sigma) + 2 \sqrt{tr(\Sigma^2) \log{\frac{1}{\delta}}} + 2 \lVert \Sigma \rVert \log{\frac{1}{\delta}} } + \sqrt{g(\hat{\beta}^{(1)}, \hat{\beta}^{(2)})} \right\}^2 
\stepcounter{equation}\tag{\theequation}\label{aeq8}
\end{align*} with probability at least $1-\delta$. With \eqref{aeq8}, the bound on $g(\beta^{(1)}, \beta^{(2)})$ finally becomes
\begin{align*}
    g(\beta^{(1)}, \beta^{(2)})
    &\leq \bigg\{ c_0 \sqrt{ tr(\Sigma) + 2 \sqrt{tr(\Sigma^2) \log{\frac{1}{\delta}}} + 2 \lVert \Sigma \rVert \log{\frac{1}{\delta}} } + \sqrt{g(\hat{\beta}^{(1)}, \hat{\beta}^{(2)})} \bigg\}^2 \\
\end{align*} 
with probability at least $1-3\delta$.
\end{proof}

Finally, the proof of Lemma \ref{lemma1} is a natural implication of Lemma \ref{alemma3} and \ref{alemma4}.
\begin{replemma}{lemma1}
(restated) With probability at least $1-5\delta$, we have $f(\sigma^2) \geq g(\beta^{(1)}, \beta^{(2)})$ if 
\begin{align*} 
    \phi_\delta (\hat{\beta}^{(1)}, \hat{\beta}^{(2)}, \hat{\sigma}^{(c) 2}) 
    \geq \psi_\delta(\hat{\beta}^{(1)}, \hat{\beta}^{(2)}, \hat{\sigma}^{(c) 2}) .
\end{align*}
\end{replemma}
\begin{proof}
From Lemma \ref{alemma3},
\begin{align*}
    h(\sigma^2) \geq \phi_\delta (\hat{\beta}^{(1)}, \hat{\beta}^{(2)}, \hat{\sigma}^{(c) 2}) 
\end{align*}
with probability at least $1-2\delta$. From Lemma \ref{alemma4},
\begin{align*}
    g(\beta^{(1)}, \beta^{(2)}) \leq \psi_\delta(\hat{\beta}^{(1)}, \hat{\beta}^{(2)}, \hat{\sigma}^{(c) 2})
\end{align*}
with probability at least $1-3\delta$. Therefore, if 
\begin{align*} 
    \phi_\delta (\hat{\beta}^{(1)}, \hat{\beta}^{(2)}, \hat{\sigma}^{(c) 2}) 
    \geq \psi_\delta(\hat{\beta}^{(1)}, \hat{\beta}^{(2)}, \hat{\sigma}^{(c) 2}) ,
\end{align*}
then $f(\sigma^2) \geq g(\beta^{(1)}, \beta^{(2)})$ with probability at least $1-5\delta$.
\end{proof}

\subsubsection{Theoretical Guarantee for Consistency of Algorithm \ref{alg1}}
In this section, we justify the consistency of our estimation on $\widehat{\text{OSE}}$ for Algorithm \ref{alg1}. Proposition \ref{prop1} guarantees that if we combine the two datasets $D_1$ and $D_2$ whenever $\phi_\delta > \psi_\delta$ and do not merge otherwise, then we have $\sum_{k=1}^2 \text{OSE}(\hat{\beta}^{(k)}, P^{(k)}_{\mathcal{X} \times \mathcal{Y}}) > 
\sum_{k=1}^2 \text{OSE}(\hat{\beta}^{(c)}, P^{(k)}_{\mathcal{X} \times \mathcal{Y}})$ with probability $1-5\delta$. Next we will show that the estimator $\widehat{\text{OSE}}$ used in Algorithm \ref{alg1} is consistent.

We aim to develop algorithm with high $\text{SR}$, which is computationally infeasible if we only have access to empirical distribution. From \eqref{eqn:SR}, a natural idea is to estimate the out-of-sample errors to approximate SR. With the lack of generating distribution, one can only rely on the existing sampled dataset $\{D_{all}\}$, and the problem becomes how to provide a valid estimation for $\text{OSE}(\hat{f}, P^{(k)}_{\mathcal{X} \times \mathcal{Y}}) = E_{P^{(k)}_{\mathcal{X} \times \mathcal{Y}}}[ (\tilde{y} - \hat{f}(\tilde{x}))^2 ]$ for some function $\hat{f}$ where $(\tilde{x}, \tilde{y}) \sim P^{(k)}_{\mathcal{X} \times \mathcal{Y}}$.

In practice, our estimator $\widehat{\text{OSE}} \left(\hat{\beta}, D_k^{out}\right)$ is computed by the method of bootstrap. 
As we do not assume knowledge of the underlying distribution, this nonparametric method is coherent with our assumption. Specifically, the $m$-th bootstrap samples $\{(x_{i,m}^*, y_{i,m}^*)\}_{i=1}^{\tilde{n}_k}$ are generated from $D_k^{out}$ for $m=1, \cdots, M$. We estimate the error by $\widehat{\text{OSE}}_M(\hat{\beta}, D_k^{out}) = \frac{1}{M} \sum_{m=1}^M \frac{1}{\tilde{n}_k} \sum_{i=1}^{\tilde{n}_k} ( y_{i,m}^* -x_{i,m}^{*\top} \hat{\beta} )^2$. By boostrapping, we get a consistent estimator of the out-of-sample error.

As described below, we can have the following theoretical guarantee under some technical assumption. 
For a reminder, samples $(X^{(k)}, y^{(k)})$ are generated from the distribution $P^{(k)}_{\mathcal{X} \times \mathcal{Y}}$ and $P^{(k)}_X$ is the marginal distribution. Let $X^{(c)}$ denote the combined data $[X^{(1)\top}, X^{(2)\top}]^\top$. Moreover, let $\mu_k$ and $\Sigma_k$ be the mean and the variance of $P^{(k)}_X$.

\begin{assumption} \label{assumption2}
Denote $b_k = E[\hat{\beta}^{(c)} | X^{(c)}] - \beta^{(k)}$, and its second moment $E[b_k b_k^\top]$ to be $\Lambda_k$. We define 
$\Omega_k = \mathbb{E}[(X^{(k)\top} X^{(k)})^{-1}] $ 
and 
$ \Omega_c = \mathbb{E}[(X^{(c)\top} X^{(c)})^{-1}]. $ Finally, we assume $\Omega_k \mu_k \neq 0$ and $\sigma^2 \Omega_c \mu_k + \Lambda_k \mu_k \neq 0$.
\end{assumption}

\begin{theorem} \label{thm2}
    With Assumption \ref{assumption1} and \ref{assumption2}, we have $\widehat{\text{OSE}}_M (\hat{\beta}^{(k)}, D_k^{out}) - \text{OSE}(\hat{\beta}^{(k)}, P^k_{X \times Y}) \overset{p}{\rightarrow} 0$ as $\tilde{n}_k \rightarrow \infty$.
\end{theorem}

\subsubsection{Consistency of Errors} \label{consistency_section}

Next, we prove Theorem \ref{thm2}, the consistency of the estimator of the out-of-sample error. We first make the following assumption.
\begin{repassumption}{assumption2}
(restated) Denote $b_k = E[\hat{\beta}^{(c)} | X^{(c)}] - \beta^{(k)}$. Let $\Lambda_k = E[b_k b_k^\top]$, $\Omega_k = \mathbb{E}[(X^{(k)\top} X^{(k)})^{-1}] $ 
and $\Omega_c = \mathbb{E}[(X^{(c)\top} X^{(c)})^{-1}]$ be given. We assume $\Omega_k \mu_k \neq 0$ and $\sigma^2 \Omega_c \mu_k + \Lambda_k \mu_k \neq 0$.
\end{repassumption}

Consistency result is built upon the following theorem on central limit theorem of bootstrapping.

\begin{theorem} \label{bootstrap} \citep{gill1989}
Let $\mathcal{F}$ denote the space of all cumulative distribution functions. Assume  $X_1, \cdots, X_n$ are iid samples from a distribution $F \in \mathcal{F}$ and $X_1^*, \cdots, X_n^*$ are bootstrapped samples from their empirical distribution $\hat{F}_n$. Let us denote the empirical distribution of the bootstrapped samples as $\hat{F}_n^*$. For $\mu = E[X_1]$, if a functional $T$ on $\mathcal{F}$ satisfies $T(F) = g(\mu)$ for some continuously differentiable function $g$ and $g^\prime (\mu) \neq 0$, then
\begin{align*}
    \sup_u \left| \mathbb{P}_{\hat{F}_n}\left( \sqrt{n}(T(\hat{F}_n^*) - T(\hat{F}_n) \leq u \right) - \mathbb{P}_{F}\left( \sqrt{n}(T(\hat{F}_n) - T(F) \leq u \right) \right| \rightarrow 0
\end{align*}
almost surely as $n \rightarrow \infty$.
\end{theorem}

Theorem \ref{thm2} is proved by comparing the limit of estimated out-of-sample error of $\hat{\beta}^{(k)}$ on original samples and bootstrapped samples.

\begin{reptheorem}{thm2}
(restated) Under Assumption \ref{assumption1} and \ref{assumption2},
\begin{align*}
    \widehat{\text{OSE}}_M (\hat{\beta}^{(k)}, D_k^{out}) - \text{OSE}(\hat{\beta}^{(k)}, P^{(k)}_{\mathcal{X} \times \mathcal{Y}}) \overset{p}{\rightarrow} 0
\end{align*}
as $\tilde{n}_k \rightarrow \infty$.
\end{reptheorem}
\begin{proof}  
The out-of-sample error of $\hat{\beta}^{(k)}$ on $P_{X \times Y}^{(k)}$ can be rewritten as
\begin{align*}
    \text{OSE}(\hat{\beta}^{(k)}, P^k_{X \times Y}) 
    &= \sigma^2 + \sigma^2 tr( E[\tilde{x}^{(k)} \tilde{x}^{(k)\top}] E[(X^{(k)\top} X^{(k)})^{-1}]) \\
    &= \sigma^2 + \sigma^2 tr( Var(\tilde{x}^{(k)}) \Omega_k) + \sigma^2 tr( E[\tilde{x}^{(k)}] E[\tilde{x}^{(k)}]^\top \Omega_k),
\end{align*}
where $\tilde{x}^{(k)}$ denotes an out-of-sample from the marginal distribution $P_{\mathcal{X}}^{(k)}$. Note that $\text{OSE}(\hat{\beta}^{(k)}, P^k_{X \times Y})$ is differentiable with respect to $E[\tilde{x}^{(k)}]$ and $\frac{d \text{OSE}(\hat{\beta}^{(k)}, P^k_{X \times Y})}{d E[\tilde{x}^{(k)}]} = 2 \sigma^2 \Omega_k E[\tilde{x}^{(k)}] \neq 0$.

The estimated $\text{OSE}$ built upon the out-of-samples $\{(\tilde{x}^{(k)}_i, \tilde{y}^{(k)}_i)\}_{i=1}^{\tilde{n}_k}$ from $D_k^{out}$ is $\widehat{\text{OSE}}(\hat{\beta}^{(k)}, D_k^{out}) = \frac{1}{\tilde{n}_k} \sum_{i=1}^{\tilde{n}_k} (\tilde{y}_i^{(k)} - \tilde{x}_i^{(k)} \hat{\beta}^{(k)} )^2$ . Let us choose $\tilde{n}_k = n_k$. By the central limit theorem, $\sqrt{n_k}( \widehat{\text{OSE}}(\hat{\beta}^{(k)}, D_k^{out}) - \text{OSE}(\hat{\beta^{(k)}}, P^{(k)}_{\mathcal{X} \times \mathcal{Y}}) ) \overset{d}{\rightarrow} N(0,1)$ as $n_k \rightarrow \infty$, hence $\widehat{\text{OSE}}(\hat{\beta}^{(k)}, D_k^{out}) - \text{OSE}(\hat{\beta^{(k)}}, P^{(k)}_{\mathcal{X} \times \mathcal{Y}}) \overset{p}{\rightarrow} 0$. 

The estimated $\text{OSE}$ on the bootstraps $\{ \{(x^*_{i,m}, y^*_{i,m})\}_{i=1}^{\tilde{n}_k} \}_{m=1}^M$ from $D_k^{out}$ is $\widehat{\text{OSE}}_M(\hat{\beta}^{(k)}, D_k^{out}) = \frac{1}{M} \sum_{m=1}^M \frac{1}{\tilde{n}_k} \sum_{i=1}^{\tilde{n}_k} (y^*_{i,m} - x^{*\top}_{i,m} \hat{\beta}^{(k)})^2$.
By Theorem \ref{bootstrap}, $\sqrt{\tilde{n}_k} \left( \widehat{\text{OSE}}_M(\hat{\beta}^{(k)}, D_k^{out}) - \widehat{\text{OSE}}(\hat{\beta}^{(k)}, D_k^{out}) \right)$ and $\sqrt{n_k}( \widehat{\text{OSE}}(\hat{\beta}^{(k)}, D_k^{out}) - \text{OSE}(\hat{\beta^{(k)}}, P^{(k)}_{\mathcal{X} \times \mathcal{Y}}) )$ asymptotically have the same distribution almost surely as $\tilde{n}_k \rightarrow \infty$. It follows that $\widehat{\text{OSE}}_M (\hat{\beta}^{(k)}, D_k^{out}) - \text{OSE}(\hat{\beta}^{(k)}, P^{(k)}_{\mathcal{X} \times \mathcal{Y}}) \overset{p}{\rightarrow} 0$.
\end{proof}

Additionally, Theorem \ref{thm3} establishes the consistency of the estimated out-of-sample error of $\hat{\beta}^{(c)}$.

\begin{theorem} \label{thm3}
Under Assumption \ref{assumption1} and \ref{assumption2},
\begin{align*}
    \widehat{\text{OSE}}_M (\hat{\beta}^{(c)}, D_k^{out}) - \text{OSE}(\hat{\beta}^{(c)}, P^{(k)}_{\mathcal{X} \times \mathcal{Y}}) \overset{p}{\rightarrow} 0
\end{align*}
as $\tilde{n}_k \rightarrow \infty$.
\end{theorem}
\begin{proof}
The out-of-sample error of $\hat{\beta}^{(c)}$ on $P^{(k)}_{\mathcal{X} \times \mathcal{Y}}$ can be rewritten as 
\begin{align*}
    \text{OSE}(\hat{\beta}^{(k)}, P^k_{X \times Y}) 
    & = \sigma^2 + \sigma^2 tr( E[\tilde{x}^{(k)} \tilde{x}^{(k)\top}] E[(X^{(c)\top} X^{(c)})^{-1}] ) + (\beta^{(1)} - \beta^{(2)})^\top E[ X^{(3-k)\top}X^{(3-k)} (X^{(c)\top}X^{(c)})^{-1}   \\
    & \quad \times E[\tilde{x}^{(k)} \tilde{x}^{(k)\top}] (X^{(c)\top}X^{(c)})^{-1} X^{(3-k)\top}X^{(3-k)}] (\beta^{(1)}-\beta^{(2)})    \\
    & = \sigma^2 + \sigma^2 tr(Var(\tilde{x}^{(k)}) \Omega_k) + \sigma^2 E[\tilde{x}^{(k)}]^\top \Omega_k E[\tilde{x}^{(k)}] + tr(E[\tilde{x}^{(k)} \tilde{x}^{(k)\top}] E[(X^{(c)\top}X^{(c)})^{-1} \\
    & \quad \times  X^{(3-k)\top}X^{(3-k)} (\beta^{(1)}-\beta^{(2)}) (\beta^{(1)} - \beta^{(2)})^\top X^{(3-k)\top}X^{(3-k)} (X^{(c)\top}X^{(c)})^{-1}]).
\end{align*}
As 
\begin{align*}
    (X^{(c)\top} X^{(c)})^{-1} X^{(3-k)\top} X^{(3-k)} (\beta^{(3-k)}-\beta^{(k)}) 
    &= \sum_{j=1}^2 (X^{(c)\top}X^{(c)})^{-1} X^{(j)\top} X^{(j)} \beta^{(j)} - \beta^{(k)} \\
    &= E[\hat{\beta}^{(c)} | X^{(c)}] - \beta^{(k)} \\
    &= b_k,
\end{align*}
the above expression becomes
\begin{align*}
    \text{OSE}(\hat{\beta}^{(k)}, P^k_{X \times Y}) 
    &= \sigma^2 + \sigma^2 tr(Var(\tilde{x}^{(k)}) \Omega_k) + \sigma^2 E[\tilde{x}^{(k)}]^\top \Omega_k E[\tilde{x}^{(k)}] + tr(Var(\tilde{x}^{(k)})  \Lambda_k) + E[\tilde{x}^{(k)}]^\top \Lambda_k E[\tilde{x}^{(k)}].
\end{align*}
Lastly, $\text{OSE}(\hat{\beta}^{(k)}, P^k_{X \times Y})$ is differentiable with respect to $E[\tilde{x}^{(k)}]$ and $\frac{d \text{OSE}(\hat{\beta}^{(k)}, P^k_{X \times Y})}{dE[\tilde{x}^{(k)}]} = 2\sigma^2 \Omega_k E[\tilde{x}^{(k)}] + 2\Lambda_k E[\tilde{x}^{(k)}] \neq 0$. By the same reasoning as in Theorem \ref{thm2}, it can be shown that $\widehat{\text{OSE}}_M (\hat{\beta}^{(c)}, D_k^{out}) - \text{OSE}(\hat{\beta}^{(c)}, P^{(k)}_{\mathcal{X} \times \mathcal{Y}}) \overset{p}{\rightarrow} 0$.
\end{proof}

\subsubsection{Overparametrized Regime}
We conclude the discussion on the regression problem by adding a remark on overparameterized regime. In regression, the overparameterized regime refers to the case where the feature size $p$ exceeds the number of observations $n$. Equivalently, the covariate matrix $X$ is in $\mathbb{R}^{n \times p}$ with $n<p$.
In this scenario, the covariance matrix $X^\top X$ is singular and, in turn, the OLS estimator is not well defined. 

One common way to bypass the issue is to use the Moore-Penrose pseudoinverse instead of the matrix inverse. Let us denote the Moore-Penrose pseudoinverse of a matrix $A$ by $A^+$. Analogous to the OLS estimator, we define Moore-Penrose (MP) estimator as $\hat{\beta}_{MP} = (X^\top X)^+ X^\top y$. Assuming $X$ has full row rank, the MP estimator can also be written as $\hat{\beta}_{MP} = X^\top (XX^\top)^{-1} y$. With this formulation, similar arguments can be made about the overparametrized regime. For instance, under Assumption \ref{assumption1}, the out-of-sample error of the MP estimator on a single distribution is 
\begin{align*}
    \text{OSE}(\hat{\beta}_{MP}^{(k)}, P_{\mathcal{X}\times\mathcal{Y}}^{(k)}) 
    &= \sigma^2 + \sigma^2 tr(E[\tilde{x}^{(k)} \tilde{x}^{(k)\top}]E[(X^{(k)\top} X^{(k)})^+]) \\
    &\quad + tr(E[\tilde{x}^{(k)} \tilde{x}^{(k)\top}] E[(I-(X^{(k)\top} X^{(k)})^+ X^{\top(k)} X^{(k)})\beta^{(k)} \beta^{(k)\top} (I-X^{(k)\top} X^{(k)} (X^{(k)\top} X^{(k)})^+)]).
\end{align*}
Compared to the OLS estimator, adopting the MP estimator incurs an additional term $tr(E[\tilde{x}^{(k)} \tilde{x}^{(k)\top}] E[(I-(X^{(k)\top} X^{(k)})^+ X^{\top(k)} X^{(k)})\beta^{(k)} \beta^{(k)\top} (I-X^{(k)\top} X^{(k)} (X^{(k)\top} X^{(k)})^+)])$. This additional term reduces to $0$ if $X^\top X$ is invertible, under which condition the pseudoinverse $(X^\top X)^+$ coincides with the matrix inverse $(X^\top X)^{-1}$ and the MP estimator $\hat{\beta}_{MP}$ becomes the OLS estimator. OSE on the combined dataset is also modified correspondingly, and parallel arguments can be made about the OSE comparison. We leave further details to future work and now turn to classification problem.

\subsection{Classification} \label{classificaion_section}
In contrast to the regression where the out-of-sample error could be precisely computed, estimators for classification usually do not have an explicit form, let alone the loss. Instead of attempting to find the exact error, we bound the out-of-sample error for classification with high probability. Based on the bounds, we propose an equivalent condition under which the error bound decreases when datasets are merged.

In this section, we consider the problem of combining two datasets on a classification task with a focus on neural networks. A classification dataset is composed of a feature of a data $x$ and the corresponding class label $y$. For a new data with an unknown label, the goal of classification is to correctly predict the label. It is common in practice that multiple datasets are available to train a model, but is not immediately clear how to leverage the datasets to their full potential. The natural question that arises is that when the datasets can be safely combined.

A criterion for deciding when to merge datasets with theoretical support would be highly appreciated as there is a tradeoff between generalization and approximation. For a brief illustration, suppose two datasets are given. On one hand, regarding all the data of the two datasets as observations from the same distribution would bring one large dataset, which is a union of the two datasets. A single model will be trained on the combined dataset, and this same model will make predictions on both of the datasets. The increased number of samples will reduce generalization error and help the model approach the best model among all the candidates. However, approximation error could be large as the same model is fitted to both of the datasets, hence restricting the model flexibility or expressivity. This tradeoff might be detrimental especially if the two datasets were sampled from significantly heterogeneous sources. On the other hand, training two models separately on each dataset would result in models more tailored to individual datasets. Nevertheless, the models would suffer from comparatively larger generalization error due to smaller number of samples, rendering the models to be far from optimal. 

To capture and better understand what happens in common across multiple neural networks for classification tasks, we employ linear models in our theoretical analysis. This choice of model is motivated from the commonalities among numerous neural networks. Modern neural networks have assorted architectures, ranging from convolutional neural networks to Transformers. Their variation can be as simple as 
just having different size of dimensions or the number of layers, or can be as complicated as using different type of layers and connections. While there are countless number of possible designs, one thing in common among most networks for classification is that they have a linear layer at the end. Output of the linear layer followed by a softmax function represents a predicted probability distribution over classes. In order to analyze various neural networks in a more abstract setting, we focus on classifcation with linear models.

Our main result is as follows. We propose and prove error bounds of the models on individual datasets and combined dataset. Then we present a theorem that provides an equivalence condition under which the error bound is minimized if the datasets are combined. The implication of the theorem is that combining the datasets culminates in reduced error bound if the underlying distributions of the two datasets are close to each other. This interpretation is consistent with what is intuitively expected about the problem. If the datasets have similar characteristics or they are generated from similar distributions, merging them would definitely be beneficial as the model could enjoy the enlarged size of the dataset.

\subsubsection{Error Bound on a Single Dataset} \label{section:classification_single}
Let us formulate the problem more formally. Let $\mathcal{X}$ be a feature space and $\mathcal{Y} = \{1, \cdots, C\}$ be a set of all possible categories. Suppose a dataset $\mathcal{D} = \{(x_i, y_i)\}_{i=1}^n$ is given. A family of linear models $\beta=\{\beta_c\}_{c=1}^C$ takes as input the feature $x_i$ and generates as output a probability distribution $softmax(\beta_1^\top x_i, \cdots, \beta_C^\top x_i)$. The model is trained with cross entropy loss, which is a common practice in classification tasks. The cross entropy loss of a model $\hat{\beta}$ augmented by regularization is defined as 
\begin{align*}
    l_\lambda(\hat{\beta}; x, y) = \sum_{c=1}^C -y_c \log{\hat{y}_c} + \frac{\lambda}{2} \sum_{c=1}^C \lVert \hat{\beta}_c \rVert^2,
\end{align*} where $\hat{y}_c$ denotes the predicted probability of the class $c$.  We slightly abuse the notation and view $y$ as either an integer or an one hot vector indicating a class when there is no confusion. Specifically, if the class of a data is $c$, then $y$ could be either $c$ or one hot vector with only the $c$th entry being $1$, depending on the context.
For the purpose of theoretical analysis, we make separability assumption of the data.
\begin{assumption}\label{assumption_classification_1}
There exists a margin $\gamma_0>0$ such that for all data $(x_i, y_i)$, if the label of $y_i$ is $c_i$, then $\beta_{c_i}^\top x_i<0$ and $\beta_{c}^\top x_i > \gamma_0$ for $\forall c \neq c_i$. In other words,
\begin{align*}
    y_i = \frac{1}{2} \left( {\bf 1}_C -[sign(\beta_1^\top x_i), \cdots, sign(\beta_C^\top x)]^\top \right)
\end{align*} 
where ${\bf 1}_C \in \mathbb{R}^C$ denotes a column vector with all entries being $1$.
\end{assumption}
Let us concatenate all the true parameters as $\beta = [\beta_1^\top, \cdots, \beta_C^\top]^\top$. A family of linear models $\hat{\beta} = [\hat{\beta}_1^\top, \cdots, \hat{\beta}_C^\top]^\top$ is trained to estimate the true value of $\beta$ and forecast label of a new data. Total $C$ linear models, followed by a softmax function, predicts the probability of the data belonging to each category.

We also introduce a surrogate loss to approximate the cross entropy loss. We define the surrogate loss as 
\begin{align*}
    l_{\lambda, \gamma}(\hat{\beta}, \beta; x) = \sum_{c=1}^C \lVert \beta_c - \hat{\beta}_c \rVert \lVert x \rVert + \sum_{c=1}^C -\beta_c^\top x g_\gamma (\beta_c^\top x) + \log{\sum_{c=1}^C e^{\hat{\beta}_c^\top x}} + \frac{\lambda}{2} \sum_{c=1}^C \lVert \hat{\beta}_c \rVert^2,
\end{align*}
where $g_\gamma (t) = \mathbbm{1}_{(-\infty, 0)} + \mathbbm{1}_{[0,\gamma]}(1-\frac{t}{\gamma})$ is the well known ramp loss.  The surrogate loss $l_{\lambda, \gamma}$ serves as a continuous bound of the cross entropy with appropriate choice of $\gamma$. Namely, 
\begin{align*}
    l_\lambda(\hat{\beta}; x, y) \leq l_{\lambda, \gamma}(\hat{\beta}, \beta; x)
\end{align*} 
if $(x,y) \sim \mathcal{M}(\beta)$ and $\gamma \leq \gamma_0$. The corresponding empirical loss and population loss are defined as 
\begin{align*}
    \hat{L}_{\lambda, \gamma}(\hat{\beta}, \beta; \mathcal{D}) 
    &= \frac{1}{n} \sum_{i=1}^n l_{\lambda, \gamma} (\hat{\beta}, \beta; x_i), \\
    L_{\lambda,\gamma}(\hat{\beta}, \beta) 
    &= E_{\tilde{x}}[l_{\lambda,\gamma}(\hat{\beta}, \beta; \tilde{x})],
\end{align*} respectively. Population loss of $l_\lambda$ is defined as 
\begin{align*}
    L_{\lambda}(\hat{\beta}, \beta) = E_{\tilde{x}, \tilde{y}}[l_{\lambda}(\hat{\beta}; \tilde{x}, \tilde{y})].
\end{align*} 
It suffices to bound $L_{\lambda,\gamma}$ in order to find an upper bound of the performance $L_{\lambda}$ of a model. 

Given two datasets $\mathcal{D}_1 = \{(x_i^{(1)}, y_i^{(1)})\}_{i=1}^{n_1}$ and $\mathcal{D}_2 = \{(x^{(2)}_i, y^{(2)}_i)\}_{i=1}^{n_2}$, our goal is to determine whether training a single family of linear models on the joint dataset would be beneficial, compared to fitting two families separately on each dataset. Namely, we would like to compare the performance of a model $\hat{\beta}^{(m)}$ trained on a single dataset $\mathcal{D}_m$ for $m=1,2$ and a model $\hat{\beta}^{(1,2)}$ trained on the joint dataset $\mathcal{D}_1 \cup \mathcal{D}_2$. 

It is desirable to determine whether to combine the datasets by comparing the performance of the best model based on the data. Let us denote the minimizers of the empirical loss and the population loss as 
\begin{align*}
    \hat{\beta}^{*} 
    &= \argmin{\hat{\beta}} \hat{L}_{\lambda, \gamma}(\hat{\beta},\beta;\mathcal{D}), \\
    \beta^{*} 
    &= \argmin{\hat{\beta}} L_{\lambda, \gamma}(\hat{\beta}, \beta),
\end{align*} respectively, where $\beta$ is the true parameter generating $\mathcal{D}$. Note that such minimizers exist as the loss is nonnegative and convex, and are unique due to the regularization. Unfortunately, finding the best models precisely is intractable as the minimization problems do not have closed form solutions, opposed to the regression problem. In practice, gradient descent is one of the popular choice to approximate the optimizer. As the surrogate loss is subdifferentiable, we employ subgradient method to estimate the optimal solution. Starting with an initial estimate $\hat{\beta}^{[1]}$, the estimator is iteratively updated by 
\begin{align*}
    \hat{\beta}^{[k+1]} = \hat{\beta}^{[k]} - \eta^{[k]} g^{[k]},
\end{align*}
where $\eta^{[k]}$ is the $k$th step size and $g^{[k]} \in \partial L_{\lambda, \gamma}(\hat{\beta}^{[k-1]}, \beta; \mathcal{D})$ is a subgradient. Instead of analyzing the performance of the empirical risk minimizer $L_{\lambda}(\hat{\beta}^{*}, \beta)$, we focus on that of the subgradient method estimator $L_{\lambda}(\hat{\beta}^{[k]}, \beta)$.

Now that the problem is well formulated, let us examine error bounds of each model. We first scrutinize excess risk of an estimator with high probability, which involves Rademacher complexity. We further bound the Rademacher complexity of the hypothesis class of interest. The subgradient estimator will then be compared with empirical risk minimizer. Joining all the results leads to complete analysis of the model performance. We start by making an additional assumption of boundedness of the feature. 
\begin{assumption} \label{assumption_classification_2}
For all data $(x_i, y_i)$, the features are bounded by $\lVert x_i \rVert \leq B$.
\end{assumption}
While there is no explicit restriction to the scope of the hypothesis class, candidate parameters naturally satisfy 
\begin{align*}
    \lVert \hat{\beta}_c \rVert \leq \sqrt{\frac{2\log{C}}{\lambda}}
\end{align*}
for all $c=1, \cdots, C$ due to the regularization. This boundedness can be easily confirmed. Suppose $\lVert \hat{\beta}_{c_0} \rVert > \sqrt{\frac{2\log{C}}{\lambda}}$ for some $c_0$. Then $l_\lambda( \hat{\beta} ) \geq \frac{\lambda}{2} \lVert \hat{\beta}_{c_0} \rVert^2 \geq \log{C} = l_\lambda (0)$, hence $\hat{\beta}$ cannot minimize the loss. Henceforth, we implicitly assume that $\hat{\beta}_c$ is bounded.

We first scrutinize the error on a single dataset $\mathcal{D} = \{(x_i, y_i)\}_{i=1}^n$ where $(x_i, y_i) \sim \mathcal{M}(\beta)$.
Under the assumptions, we can bound the difference between the empirical loss and the population loss of an estimator. Constructing a bound requires McDiarmid's inequality, which is stated below.

\begin{lemma} \label{mcdiarmid}
\citep{mcdiarmid1989} Let $f: \mathbb{R}^m \rightarrow \mathbb{R}$ be a function of bounded difference, that is, $|f(z_1, \cdots, z_m) - f(z_1, \cdots, z_i^\prime, \cdots, z_m)| \leq c_i$ for $\forall i, z_1, \cdots, z_m, z_i^\prime$. Let $Z_1, \cdots, Z_m$ be independent random variables. Then 
\begin{align*}
\mathbb{P}(f(Z_1, \cdots, Z_m) - E[f(Z_1, \cdots, Z_m)] \geq \epsilon) \leq e^{\frac{-2\epsilon^2}{\sum_{i=1}^n c_i^2}} 
\end{align*}
for all $\epsilon >0$.
\end{lemma}

We can now analyze the excess risk of an estimator. The following proposition proposes an error bound, the proof of which uses ordinary arguments in generalization bounds. See \citet{bartlett2002} for example. 

\begin{prop} \label{classification_generalization_1}
If the dataset $\mathcal{D}$ is generated by $(x_i, y_i) \sim \mathcal{M}(\beta)$, then the difference between the empirical loss and the population loss is bounded by
\begin{align*}
\hat{L}_{\lambda, \gamma}(\hat{\beta},\beta;\mathcal{D}) - L_{\lambda, \gamma}(\hat{\beta}, \beta) 
\leq \sqrt{\frac{1}{2n}\log{\frac{1}{\delta}}}\omega(\beta) + 2 R_{n} (l_{\lambda, \gamma}) 
\end{align*} for all $\hat{\beta}$ with probability at least $1-\delta$, where $\omega(\beta) = O(\lVert\beta\rVert)$,
\begin{align*}
R_{n} (l_{\lambda, \gamma}) 
	&= E\left[\sup_{\hat{\beta}}\frac{1}{n} \sum_{i=1}^{n} \sigma_i l_{\lambda, \gamma} (\hat{\beta},\beta;x_i)\right]
\end{align*}
is Rademacher complexity, and $\sigma_i$ is Rademacher random variable.
\end{prop}
\begin{proof}
Let $
\omega(\beta) = 2 C B \sqrt{\frac{\log{C}}{\lambda}}+ 2 B \sum_{c=1}^C \lVert \beta_c \rVert + (C+1) \log{C}
$. 
It is straightforward that $l_{\lambda,\gamma}(\hat{\beta},\beta;x_i) \leq \omega(\beta)$. Let 
$f(\mathcal{D}) = \sup_{\hat{\beta}} \left\{ \hat{L}_{\lambda, \gamma}(\hat{\beta},\beta;\mathcal{D}) - L_{\lambda, \gamma}(\hat{\beta}, \beta) \right\}$. Then
\begin{align*}
| f(\mathcal{D}) - f_m(x_1,y_1,\cdots,x^{\prime}_i,y^{\prime}_i,\cdots,x_{n},y_{n}) |
& \leq \bigg| \sup_{\hat{\beta}} \left\{ \hat{L}_{\lambda, \gamma}(\hat{\beta},\beta;\mathcal{D})  
    - L_{\lambda, \gamma}(\hat{\beta}, \beta) \right\} \\
&\quad - \sup_{\hat{\beta}} \left\{ \hat{L}_{\lambda, \gamma}(\hat{\beta},\beta;\mathcal{D}) 
    - L_{\lambda, \gamma}(\hat{\beta}, \beta) + \frac{1}{n} l_{\lambda,\gamma}(\hat{\beta},\beta;x^{\prime}_i) - \frac{1}{n} l_{\lambda,\gamma}(\hat{\beta},\beta;x^{\prime}_i)\right\} \bigg| \\
&\leq \frac{\omega(\beta)}{n}.
\end{align*}
Hence 
\begin{align*}
\mathbb{P}(f(\mathcal{D}) - E[f(\mathcal{D})] \geq \epsilon) \leq e^{-\frac{2n}{\omega(\beta)^{2}}\epsilon^2}
\end{align*}
for all $\epsilon>0$.
Moreover, the expectation can be bounded by the Rademacher complexity. Let $\mathcal{D}^{\prime} = \{(x_i^{\prime}, y_i^{\prime})\}_{i=1}^n$ be a set of independent copies of each $(x_i, y_i)$. Then
\begin{align*}
E[f(\mathcal{D})]
&= E\left[\sup_{\hat{\beta}} \left\{ \hat{L}_{\lambda, \gamma}(\hat{\beta},\beta;\mathcal{D}) - L_{\lambda, \gamma}(\hat{\beta}, \beta) \right\}\right] \\
&= E\left[\sup_{\hat{\beta}} \left\{ \hat{L}_{\lambda, \gamma}(\hat{\beta},\beta;\mathcal{D}) - E[\hat{L}_{\lambda, \gamma}(\hat{\beta},\beta;\mathcal{D}^{\prime})] \right\}\right] \\
&= E\left[\sup_{\hat{\beta}} E\left[ \hat{L}_{\lambda, \gamma}(\hat{\beta},\beta;\mathcal{D}) - \hat{L}_{\lambda, \gamma}(\hat{\beta},\beta;\mathcal{D}^{\prime}) | \mathcal{D} \right] \right] \\
& \leq E\left[\sup_{\hat{\beta}} \left\{ \hat{L}_{\lambda, \gamma}(\hat{\beta},\beta;\mathcal{D}) - \hat{L}_{\lambda, \gamma}(\hat{\beta},\beta;\mathcal{D}^{\prime}) \right\} \right] \\
&= E\left[\sup_{\hat{\beta}} \frac{1}{n} \sum_{i=1}^{n} \left\{ l_{\lambda,\gamma}(\hat{\beta},\beta; x_i) - l_{\lambda,\gamma}(\hat{\beta},\beta; x_i^{\prime}) \right\} \right] \\
&= E\left[\sup_{\hat{\beta}} \frac{1}{n} \sum_{i=1}^{n} \sigma_i \left\{ l_{\lambda,\gamma}(\hat{\beta},\beta; x_i) - l_{\lambda,\gamma}(\hat{\beta},\beta; x_i^{\prime}) \right\} \right] \\
&\leq E\left[\sup_{\hat{\beta}} \frac{1}{n} \sum_{i=1}^{n} \sigma_i  l_{\lambda,\gamma}(\hat{\beta},\beta; x_i) + \sup_{\hat{\beta}} \frac{1}{n} \sum_{i=1}^{n} -\sigma_i   l_{\lambda,\gamma}(\hat{\beta},\beta; x_i^{\prime}) \right]\\
&= 2 R_{n} (l_{\lambda, \gamma}).
\end{align*}
Therefore,
\begin{align*}
\mathbb{P}\left( f(\mathcal{D}) \geq \epsilon \right) 
&\leq e^{-\frac{2n}{\omega(\beta)^2}(\epsilon - E[f(\mathcal{D})])^2} \\
&\leq e^{-\frac{2n}{\omega(\beta)^2}(\epsilon - 2 R_{n} (l_{\lambda, \gamma}))^2} .
\end{align*}
Setting $\delta = e^{-\frac{2n}{\omega(\beta)^2}(\epsilon - 2 R_{n} (l_{\lambda, \gamma}))^2}$ gives the result.
\end{proof}

Similar result on the difference between the population loss and the empirical loss can be obtained by the same argument. We state the result and omit the proof.
\begin{prop} \label{classification_generalization_2}
Under the same assumption of Proposition \ref{classification_generalization_1}, the difference between the population loss and the empirical loss is bounded by
\begin{align*}
L_{\lambda, \gamma}(\hat{\beta}, \beta) - \hat{L}_{\lambda, \gamma}(\hat{\beta},\beta;\mathcal{D}) 
\leq \sqrt{\frac{1}{2n}\log{\frac{1}{\delta}}} \omega(\beta) + 2 R_{n} (l_{\lambda, \gamma})
\end{align*} for all $\hat{\beta}$ with probability at least $1-\delta$.
\end{prop}

Previous lemmas bound the difference between $L_{\lambda,\gamma}$ and $\hat{L}_{\lambda,\gamma}$ with high probability using the Rademacher complexity. While the bounds provide some insight into the excess risk, it would give rise to better comprehension if the Rademacher complexity is replaced by more interpretable terms. To this end, we further bound the Rademacher complexity.
The following two lemmas will be useful when we analyze the Rademacher complexity.

\begin{lemma} \label{combinatorial_bound}
For $n \in \mathbb{N}$,
\begin{align*}
\frac{1}{2^n} \sum_{k=0}^{\left[\frac{n}{2}\right]} (n-2k) \binom{n}{k} \leq \sqrt{\frac{3n}{4\pi}} .
\end{align*}
\end{lemma}
\begin{proof}
If $n$ is even,
\begin{align*}
\frac{1}{2^n} \sum_{k=0}^{\left[\frac{n}{2}\right]} (n-2k) \binom{n}{k} 
&= \frac{n}{2^n} \sum_{k=0}^{\frac{n}{2}} \binom{n}{k} - \frac{n}{2^{n-1}} \sum_{k=1}^{\frac{n}{2}} \binom{n-1}{k-1} \\
&= \frac{n}{2^n}\left(\frac{2^n - \binom{n}{\frac{n}{2}}}{2} +  \binom{n}{\frac{n}{2}} \right) - \frac{n}{2^{n-1}} 2^{n-2} \\
&= \frac{n}{2^{n+1}} \binom{n}{\frac{n}{2}} \\
&\leq \frac{n}{2^{n+1}} \frac{\sqrt{2\pi n } \left(\frac{n}{e}\right)^n e^{\frac{1}{12n}}}{2\pi \frac{n}{2} \left(\frac{n}{2e}\right)^n e^{\frac{2}{6n+1}}} \\
&\leq \sqrt{\frac{n}{2\pi}}.
\end{align*}
If $n$ is odd,
\begin{align*}
\frac{1}{2^n} \sum_{k=0}^{\left[\frac{n}{2}\right]} (n-2k) \binom{n}{k} 
&= \frac{n}{2^n} \sum_{k=0}^{\frac{n-1}{2}} \binom{n}{k} - \frac{n}{2^{n-1}} \sum_{k=1}^{\frac{n-1}{2}} \binom{n-1}{k-1} \\
&= \frac{n}{2^n}\frac{2^n}{2} - \frac{n}{2^{n-1}}\left(\frac{2^{n-1}-\binom{n-1}{\frac{n-1}{2}}}{2}\right) \\
&= \frac{n}{2^{n}} \binom{n-1}{\frac{n-1}{2}} \\
&\leq \frac{n}{2^{n}} \frac{\sqrt{2\pi (n-1) } \left(\frac{n-1}{e}\right)^{n-1} e^{\frac{1}{12n-12}}}{2\pi \frac{n-1}{2} \left(\frac{n-1}{2e}\right)^{n-1} e^{\frac{2}{6n-5}}} \\
&\leq \sqrt{\frac{3n}{4\pi}}.
\end{align*}
\end{proof}

The proof of the following lemma is motivated by \citet{ledoux2011}.
\begin{lemma} \label{lipschitz_lemma}
Let $f : \mathbb{R}^m \rightarrow \mathbb{R}$ be an $L$-Lipschitz function, $g_i: \mathbb{R} \rightarrow \mathbb{R}$ for $i=1, \cdots, n$ where $n\geq 0$, and $h: \mathbb{R}^m \rightarrow \mathbb{R}$. Then 
\begin{align*}
E \left[\sup_{\theta_1, \cdots, \theta_m} h(\theta_1,\cdots, \theta_m) + \sum_{i=1}^n \sigma_i f(g_i(\theta_1), \cdots, g_i(\theta_m)) \right] 
\leq E \left[\sup_{\theta_1, \cdots, \theta_m} h(\theta_1,\cdots, \theta_m) + L \sum_{i=1}^n \sigma_i \sum_{j=1}^m g_i(\theta_j) \right]
\end{align*}
\end{lemma}
\begin{proof}
For $n=0$, the inequality holds. Assume that the inequality holds for $n-1$. Then
\begin{align*}
E & \left[\sup_{\theta_1, \cdots, \theta_m} \sum_{i=1}^n h(\theta_1,\cdots, \theta_m) + \sigma_i f(g_i(\theta_1), \cdots, g_i(\theta_m)) \right] \\
& = \frac{1}{2} E \left[\sup_{\theta_1, \cdots, \theta_m} h(\theta_1,\cdots, \theta_m) + \sum_{i=1}^{n-1} \sigma_i f(g_i(\theta_1), \cdots, g_i(\theta_m))  + f(g_n(\theta_1), \cdots, g_n(\theta_m) ) \right] \\
& + \frac{1}{2} E \left[\sup_{\theta_1, \cdots, \theta_m} h(\theta_1,\cdots, \theta_m) + \sum_{i=1}^{n-1} \sigma_i f(g_i(\theta_1), \cdots, g_i(\theta_m)) - f(g_n(\theta_1), \cdots, g_n(\theta_m) ) \right] \\
&= \frac{1}{2} E \Bigg[\sup_{\theta_1, \cdots, \theta_m, \theta_1^\prime, \cdots, \theta_m^\prime} h(\theta_1,\cdots, \theta_m) +h(\theta_1^\prime,\cdots, \theta_m^\prime) + \sum_{i=1}^{n-1} \sigma_i \left\{ f(g_i(\theta_1), \cdots, g_i(\theta_m)) + f(g_i(\theta_1^\prime), \cdots, g_i(\theta_m^\prime))\right\} \\
&\quad + \left| f(g_n(\theta_1), \cdots, g_n(\theta_m) ) - f(g_n(\theta_1^\prime), \cdots, g_n(\theta_m^\prime) )\right|\Bigg] \\
&\leq \frac{1}{2} E \Bigg[\sup_{\theta_1, \cdots, \theta_m, \theta_1^\prime, \cdots, \theta_m^\prime} h(\theta_1,\cdots, \theta_m) +h(\theta_1^\prime,\cdots, \theta_m^\prime) +\sum_{i=1}^{n-1} \sigma_i \left\{ f(g_i(\theta_1), \cdots, g_i(\theta_m)) + f(g_i(\theta_1^\prime), \cdots, g_i(\theta_m^\prime))\right\} \\
&\quad + L \left \lVert \begin{bmatrix} g_n(\theta_1) \\ \vdots \\ g_n(\theta_m) \end{bmatrix} - \begin{bmatrix} g_n(\theta_1^\prime) \\ \vdots \\ g_n(\theta_m^\prime) \end{bmatrix} \right\rVert \Bigg] \\
&\leq \frac{1}{2} E \Bigg[\sup_{\theta_1, \cdots, \theta_m, \theta_1^\prime, \cdots, \theta_m^\prime} h(\theta_1,\cdots, \theta_m) +h(\theta_1^\prime,\cdots, \theta_m^\prime) +\sum_{i=1}^{n-1} \sigma_i \left\{ f(g_i(\theta_1), \cdots, g_i(\theta_m)) + f(g_i(\theta_1^\prime), \cdots, g_i(\theta_m^\prime))\right\} \\
&\quad + L \sum_{j=1}^m |g_n(\theta_j) - g_n(\theta_j^\prime)| \Bigg] \\
&= E \left[\sup_{\theta_1, \cdots, \theta_m} h(\theta_1,\cdots, \theta_m) + \sum_{i=1}^{n-1} \sigma_i  f(g_i(\theta_1), \cdots, g_i(\theta_m)) + L \sigma_n \sum_{j=1}^m g_n(\theta_j)\right] \\
&\leq E \left[\sup_{\theta_1, \cdots, \theta_m} h(\theta_1,\cdots, \theta_m) +  L \sum_{i=1}^n \sigma_i \sum_{j=1}^m g_i(\theta_j) \right] \text{by induction hypothesis.}
\end{align*}
Hence the inequality holds for all $n\geq0$.
\end{proof}

Now we can bound the Rademacher complexity of the hypothesis class of interest with the help of the preceding lemmas. Let us define $\vertiii{\beta} = \sum_{c=1}^C \lVert \beta_c \rVert$ to further simplify the notation.

\begin{prop} \label{rademacher_bound_1}
The Rademacher complexity is bounded by
\begin{align*}
R_{n} (l_{\lambda, \gamma})  
	\leq \frac{1}{2} \sqrt{\frac{3}{\pi n}} B\left(\sqrt{\frac{2\log{C}}{\lambda}}C + \vertiii{\beta} \right)
	+ \frac{BC}{\sqrt{n}} \sqrt{\frac{2\log{C}}{\lambda}} + \frac{1}{2}\sqrt{\frac{3}{\pi n}}\log{C}
\end{align*}
\end{prop}
\begin{proof}
By definition,
\begin{align*}
R_{n} (l_{\lambda, \gamma}) 
&= E[\sup_{\hat{\beta}}\frac{1}{n} \sum_{i=1}^{n} \sigma_i l_{\lambda, \gamma} (\hat{\beta},\beta;x_i)] \\
&= E \left[\sup_{\hat{\beta}} \frac{1}{n} \sum_{i=1}^n \sigma_i \left\{
	\sum_{c=1}^C \lVert \beta_c - \hat{\beta}_c \rVert \lVert x_i \rVert 
	+ \sum_{c=1}^C -\beta_c^{\top} x_i g_\gamma (\beta_c^{\top} x_i) 
	+ \log{\sum_{c=1}^C e^{\hat{\beta}_c^\top x_i}} 
	+ \frac{\lambda}{2} \sum_{c=1}^C \lVert \hat{\beta}_c \rVert^2 \right\} \right] \\
&\leq E \left[\sup_{\hat{\beta}} \frac{1}{n} \sum_{i=1}^n \sigma_i \sum_{c=1}^C \lVert \beta_c - \hat{\beta}_c \rVert \lVert x_i \rVert \right] 
	+ E \left[\sup_{\hat{\beta}} \frac{1}{n} \sum_{i=1}^n \sigma_i \sum_{c=1}^C -\beta_c^{\top} x_i g_\gamma (\beta_c^{\top} x_i)  \right] 
	+ E \left[\sup_{\hat{\beta}} \frac{1}{n} \sum_{i=1}^n \sigma_i \log{\sum_{c=1}^C e^{\hat{\beta}_c^\top x_i}}  \right] \\
&\quad + E \left[\sup_{\hat{\beta}} \frac{1}{n} \sum_{i=1}^n \sigma_i \frac{\lambda}{2} \sum_{c=1}^C \lVert \hat{\beta}_c \rVert^2 \right] .
\end{align*}
The first term can be bounded by
\begin{align*}
E \left[\sup_{\hat{\beta}} \frac{1}{n} \sum_{i=1}^{n} \sigma_i \sum_{c=1}^C \lVert \beta_c - \hat{\beta}_c \rVert \lVert x_i \rVert \right] 
&\leq E \left[\sup_{\hat{\beta}} \frac{1}{n} \sum_{i=1}^{n} \sigma_i \sum_{c=1}^C B \lVert \beta_c - \hat{\beta}_c \rVert  \right] \\
&= E \left[\frac{B}{n} \sum_{i=1}^{n} \sigma_i \sum_{c=1}^C \left( \lVert \beta_c \rVert + \sqrt{\frac{2\log{C}}{\lambda}} \right) \mathbbm{1}_{\sum_{i=1}^{n} \sigma_i>0} \right] \\
&= \frac{B}{n}  \left( C\sqrt{\frac{2\log{C}}{\lambda}} + \sum_{c=1}^C\lVert \beta_c \rVert \right) \frac{1}{2^{n}} \sum_{k=0}^{\left[\frac{n}{2}\right]} (n-2k) \binom{n}{k} \\
&\leq \sqrt{\frac{3}{4\pi n}} B \left( C\sqrt{\frac{2\log{C}}{\lambda}} + \sum_{c=1}^C\lVert \beta_c \rVert \right)
\end{align*}
The second term is $0$ as it does not contain any $\hat{\beta}$ and all the $x_i$'s and $\sigma_i$'s are independent. Especially, $\sigma_i$ is symmetric around $0$.

As the log-sum-exp function is a contraction with respect to $\lVert \cdot \rVert_\infty$, the third term can be bounded by 
\begin{align*}
E \left[\sup_{\hat{\beta}} \frac{1}{n} \sum_{i=1}^{n} \sigma_i \log{\sum_{c=1}^C e^{\hat{\beta}_c^\top x_i}} \right] 
&\leq E \left[\sup_{\hat{\beta}} \frac{1}{n} \sum_{i=1}^{n} \sigma_i \sum_{c=1}^C \hat{\beta}_c^\top x_i \right] \\
&= E \left[\sup_{\hat{\beta}} \frac{1}{n} \sum_{c=1}^C \hat{\beta}_c^\top \sum_{i=1}^n \sigma_i x_i \right] \\
&\leq E \left[ \frac{1}{n} \sum_{c=1}^C \sqrt{\frac{2\log{C}}{\lambda}} \lVert \sum_{i=1}^{n} \sigma_i x_i \rVert \right] \\
&\leq \frac{1}{n} \sqrt{\frac{2\log{C}}{\lambda}} \sum_{c=1}^C \sqrt{E \left[ \lVert \sum_{i=1}^{n} \sigma_i x_i \rVert^2 \right] } \\
&= \frac{1}{n} \sqrt{\frac{2\log{C}}{\lambda}} \sum_{c=1}^C \sqrt{E \left[ \sum_{i=1}^{n} \lVert  x_i \rVert^2 \right] } \\
&\leq \frac{1}{\sqrt{n}} \sqrt{\frac{2\log{C}}{\lambda}} C B \text{ using lemma \ref{lipschitz_lemma}}.
\end{align*}
The last term can be bounded by 
\begin{align*}
E \left[\sup_{\hat{\beta}} \frac{1}{n} \sum_{i=1}^{n} \sigma_i \frac{\lambda}{2} \sum_{c=1}^C \lVert \hat{\beta}_c \rVert^2 \right] 
&= E \left[\frac{\lambda}{2n} C \frac{2\log{C}}{\lambda} \mathbbm{1}_{\sum_{i=1}^{n} \sigma_i >0 } \sum_{i=1}^{n} \sigma_i \right] \\
&= \frac{C \log{C}}{n}\sum_{k=0}^{\left[\frac{n}{2}\right]} (n-2k) \frac{\binom{n}{k}}{2^{n}}\\
&\leq \frac{C\log{C}}{n}  \sqrt{\frac{3n}{4\pi}} \text{ by lemma \ref{combinatorial_bound}}.
\end{align*}
Combining all the bounds gives the desired result.
\end{proof}

As explained before, estimators are obtained through subgradient method since direct optimization of the loss is intractable. The following lemma provides a guarantee that the lowest loss among all the updates is close to the optimal loss under suitable conditions.
\begin{lemma}
\citep{shor1985} Let $f : \mathbb{R}^m \rightarrow \mathbb{R}$ be subdifferentiable with bounded subgradients such that $\lVert g\rVert \leq G$ for all $g \in \partial f(x)$. Let $x^{[k+1]} = x^{[k]} - \eta^{[k]} g^{[k]}$ where $g^{[k]} \in \partial f(x^{[k]})$ and $\eta^{[k]}$ is a step size at the $k$th step. If $f(x^*) = \min f(x)$, then 
\begin{align*}
\min\{f(x^{[k]}) | k=1, \cdots, K\} - f(x^*) \leq \frac{\lVert x^{[1]}-x^* \rVert^2 + G \sum_{k=1}^K \eta^{[k]2}}{2\sum_{k=1}^K \eta^{[k]}} .
\end{align*}
\end{lemma}

The above lemma gives the following guarantee for subgradient method estimator for the surrogate loss, which is stated in the following proposition. For simplicity, we choose to use the same step size and the number of steps for all datasets, but it is also possible to consider different step sizes and the number of steps.

\begin{prop} \label{classification_subgradient_1}
Let the subgradient method estimator on $\mathcal{D}$ be iteratively defined as 
$\hat{\beta}^{[k+1]} = \hat{\beta}^{[k]} - \eta^{[k]} g^{[k]}$ where $\hat{\beta}^{[1]}$ is an initial estimator, $\eta^{[k]}$ is the $k$th step size, and 
$g^{[k]} \in \partial \hat{L}_{\lambda, \gamma}(\hat{\beta}^{[k]}, \beta; \mathcal{D})$. 
Then
\begin{align*}
\min\{\hat{L}_{\lambda,\gamma}(\hat{\beta}^{[k]};\mathcal{D}) | k=1, \cdots, K\} - \hat{L}_{\lambda,\gamma}(\hat{\beta}^{*};\mathcal{D}) 
\leq \frac{8\log{C} + G \lambda \sum_{k=1}^K \eta^{[k]2}}{2\lambda \sum_{k=1}^K \eta^{[k]}}
\end{align*}
where $G=2CB+ C \sqrt{2\lambda \log{C}}$.
\end{prop}
\begin{proof}
Note that 
\begin{align*}
    l_{\lambda, \gamma}(\hat{\beta}, \beta; x) = \sum_{c=1}^C \lVert \beta_c - \hat{\beta}_c \rVert \lVert x \rVert + \sum_{c=1}^C -\beta_c^\top x g_\gamma (\beta_c^\top x) + \log{\sum_{c=1}^C e^{\hat{\beta}_c^\top x}} + \frac{\lambda}{2} \sum_{c=1}^C \lVert \hat{\beta}_c \rVert^2.
\end{align*}
The term $\lVert \beta_c - \hat{\beta}_c \rVert$ is subdifferentiable with respect to $\hat{\beta}_c$ at $\hat{\beta}_c = \beta_c$ with subdifferential being a unit ball and differentiable elsewhere with 
$\nabla_{\hat{\beta}_c} \lVert \beta_c - \hat{\beta}_c \rVert = \frac{\beta_c - \hat{\beta}_c}{\lVert \hat{\beta}_c-\beta_c  \rVert}$. The other terms are all differentiable with 
$\nabla_{\hat{\beta}_c} \log{\sum_{c^\prime =1}^C e^{\hat{\beta}_{c^\prime}^\top x}} = \frac{e^{\hat{\beta}_c^\top x}}{\sum_{c^\prime=1}^C e^{\hat{\beta}_{c^\prime}^\top x}}x$ and 
$\nabla_{\hat{\beta}_c} \frac{\lambda}{2} \sum_{c^\prime=1}^C \lVert \hat{\beta}_{c^\prime} \rVert^2 = \lambda \hat{\beta}_c$. As 
$\hat{L}_{\lambda, \gamma}(\hat{\beta}, \beta; \mathcal{D}) = \frac{1}{n} \sum_{i=1}^{n} l_{\lambda, \gamma} (\hat{\beta}, \beta; x_i)$, the subgradients of $\hat{L}_{\lambda,\gamma}(\hat{\beta}),\beta;\mathcal{D})$ are bounded by $G$. The result follows by applying the above lemma and recalling that both $\hat{\beta}^{[1]}$ and $\hat{\beta}^{*}$ are bounded.
\end{proof}

Integrating all the propositions leads to the main theorem, which guarantees the subgradient method estimator on a single dataset. In the following theorem, the error bound is split into several parts for better interpretation. Details will be elucidated after the proof.

\begin{theorem}
\label{thm_classification_1}
Let $\hat{\beta}^{[k]}$ be the subgradient method estimator of $\hat{L}_{\lambda, \gamma}(\hat{\beta}, \beta; \mathcal{D})$. Let $s = \argmin{k \leq K}{\hat{L}_{\lambda,\gamma}(\hat{\beta}^{[k]};\mathcal{D})}$ denote the step that achieved the lowest loss during $K$ subgradient method updates. Then
\begin{align*}
L_\lambda(\hat{\beta}^{[s]};\beta) 
&\leq \Phi(n, \beta; \delta)
	+ L_{\lambda, \gamma}(\beta^{*},\beta)
\end{align*} with probability at least $1-2\delta$ where $\Phi = O(\sqrt{\frac{1}{n}\log{\frac{1}{\delta}}}) $ and 
\begin{align*}
\beta^{*} = \argmin{\hat{\beta}} L_{\lambda, \gamma}(\hat{\beta},\beta)
\end{align*}
denotes the minimizer for the population surrogate loss.
\end{theorem}
\begin{proof}
By the analysis on the excess risk in Proposition \ref{classification_generalization_1} and \ref{classification_generalization_2},
\begin{align*}
L_{\lambda, \gamma}(\hat{\beta}^{[s]}, \beta) - \hat{L}_{\lambda, \gamma}(\hat{\beta}^{[s]},\beta;\mathcal{D}) 
\leq \sqrt{\frac{1}{2n}\log{\frac{1}{\delta}}}\omega(\beta) + 2 R_{n} (l_{\lambda, \gamma})
\end{align*}
with probability at least $1-\delta$ and 
\begin{align*}
\hat{L}_{\lambda, \gamma}(\hat{\beta}^{*(m)},\beta^{(m)};\mathcal{D}^{(m)}) - L_{\lambda, \gamma}(\hat{\beta}^{*(m)}, \beta^{(m)}) 
\leq \sqrt{\frac{1}{2n}\log{\frac{1}{\delta}}}\omega(\beta) + 2 R_{n} (l_{\lambda, \gamma})
\end{align*}
with probability at least $1-\delta$. Also,
\begin{align*}
\hat{L}_{\lambda,\gamma}(\hat{\beta}^{[s]};\mathcal{D}) - \hat{L}_{\lambda,\gamma}(\hat{\beta}^{*};\mathcal{D}) 
\leq \frac{8\log{C} + G \lambda \sum_{k=1}^K \eta^{[k]2}}{2\lambda \sum_{k=1}^K \eta^{[k]}}
\end{align*}
by the subgradient method analysis in Proposition \ref{classification_subgradient_1} and 
\begin{align*}
\hat{L}_{\lambda,\gamma}(\hat{\beta}^{*};\mathcal{D}) \leq \hat{L}_{\lambda,\gamma}(\beta^{*};\mathcal{D})
\end{align*}
by the definition of $\hat{\beta}^{*}$. Adding all the inequalities leads to 
\begin{align*}
L_{\lambda, \gamma}(\hat{\beta}^{[s]}, \beta) - L_{\lambda, \gamma}(\hat{\beta}^{*}, \beta) 
\leq 2 \sqrt{\frac{1}{2n}\log{\frac{1}{\delta}}}\omega(\beta) + 4 R_{n} (l_{\lambda, \gamma})
 	+ \frac{8\log{C} + G \lambda \sum_{k=1}^K \eta^{[k]2}}{2\lambda \sum_{k=1}^K \eta^{[k]}}
\end{align*}
with probability at least $1-2\delta$. The proof concludes by recalling that $L_{\lambda, \gamma}$ serves as an upper bound of $L_\lambda$ and applying the bound on the Rademacher complexity in Proposition \ref{rademacher_bound_1}. 
Specifically, $\Phi$ has the form of 
\begin{align*}
\Phi(n, \beta; \delta)
	&= \phi_1(n,\beta;\delta) 
	+ \phi_2(n, \beta) 
	+ \phi_3
\end{align*} where
\begin{align*}
\phi_1(n,\beta; \delta) 
	&= \sqrt{\frac{2}{n}\log{\frac{1}{\delta}}} \omega(\beta), \\
\phi_2(n, \beta) 
	&= \frac{a_1 + a_2 \vertiii{\beta} }{\sqrt{n}}, \\
\phi_3
	&= \frac{8\log{C} + G \lambda \sum_{k=1}^K \eta^{[k]2}}{2\lambda \sum_{k=1}^K \eta^{[k]}}
\end{align*}
for constants
$
a_1 = \left(2\sqrt{\frac{6}{\pi}} + 4\sqrt{2}\right) \sqrt{\frac{\log{C}}{\lambda}} BC + 2 \sqrt{\frac{3}{\pi}}C\log{C}$ and $
a_2 = 2 \sqrt{\frac{3}{\pi}} B
$.
\end{proof}

The bound $\Phi$ on the excess risk is composed with three parts. The first part, denoted as $\phi_1$ in the theorem, is the error stemming from the randomness of the samples. It is directly connected to the confidence level by being a function of $\delta$. The second part, which is $\phi_2$, is the error related to the complexity of the hypothesis class. The last part $\phi_3$ is the error due to the subgradient method.

This theorem provides full analysis of the performance of the subgradient method estimator $L_\lambda(\hat{\beta}^{[s]};\beta)$ in terms of the number of samples, distributional properties, and the step sizes of the subgradient method. Suppose the step size $\eta^{[k]}$ converges to 0 and $\sum_{k=1}^\infty \eta^{[k]} =\infty$. For instance, taking the step size to be $\eta^{[k]} = \frac{1}{k}$ satisfies both conditions. Then the error due to subgradient method converges to $0$, implying that the subgradient method converges. Under this scenario, the error bound is related to the size of the dataset $n_m$. Especially, the theorem theoretically proves that the error bound decreases if there are more samples, as one might have intuitively expected. 

The bound also depends on the size of the hypothesis class, which is controlled by the regularization coefficient $\lambda$. The familiy of candidate models grows larger as $\lambda$ decreases. The effect of $\lambda$ is hidden in the coefficients $a_1$ and $a_3$. For smaller $\lambda$, these coefficiets becomes larger, which in turn implies that the number of samples required to achieve the same error bound increases.

While the above theorem compares the performance of $\hat{\beta}^{[s](m)}$ with $\beta^{*(m)}$ on a single dataset $\mathcal{D}_m$, it could also be compared with 
\begin{align*}
    \beta^{*(1,2)} = \argmin{\hat{\beta}} \sum_{m=1}^2 L_{\lambda, \gamma}(\hat{\beta},\beta^{(m)}),
\end{align*}
the minimizer of the population surrogate loss on the combined datset. The proof uses the same line of reasoning. This another comparison will later be used to construct a criterion for merging the datasets.

\subsubsection{Error Bound on the Combined Dataset} \label{section:classification_combined}
Performance guarantee on the combined dataset can be proved in a similar way. Most parts of the proof on the combined dataset are essentially identical to the proof for the single dataset. We therefore only demonstrate the parts where the reasoning deviates essentially from the previous proof. We begin by proposing another concentration inequality on the joint loss. Note that in the following proposition, the same model $\hat{\beta}$ is applied to both datasets, whereas there was one model per one dataset previously. We also assume that features on the both datasets are bounded by the same constant $B$, but our analysis could be easily extended to the case where the bounds are different on each dataset.

\begin{prop} \label{classification_generalization_3}
The difference between the empirical loss and the population loss on $\mathcal{D}_1 \cup \mathcal{D}_2$ is bounded by
\begin{align*}
\sum_{m=1}^2 \left\{ \hat{L}_{\lambda, \gamma}(\hat{\beta},\beta^{(m)};\mathcal{D}_m) - L_{\lambda, \gamma}(\hat{\beta}, \beta^{(m)}) \right\}
\leq \sqrt{ \frac{1}{2}\left(\frac{\omega(\beta^{(1)})^2}{n_1}+\frac{\omega(\beta^{(2)})^2}{n_2}\right)\log{\frac{1}{\delta}}} + 2 R_{n_1,n_2} (l_{\lambda, \gamma})
\end{align*} for all $\hat{\beta}$ with probability at least $1-\delta$ where 
\begin{align*}
R_{n_1,n_2} (l_{\lambda, \gamma}) = E\left[\sup_{\hat{\beta}}\sum_{m=1}^2 \frac{1}{n_m} \sum_{i=1}^{n_m} \sigma^{(m)}_i l_{\lambda, \gamma} (\hat{\beta},\beta^{(m)};x^{(m)}_i)\right].
\end{align*}
Also,
\begin{align*}
\sum_{m=1}^2 \left\{ L_{\lambda, \gamma}(\hat{\beta}, \beta^{(m)}) - \hat{L}_{\lambda, \gamma}(\hat{\beta},\beta^{(m)};\mathcal{D}_m) \right\}
\leq \sqrt{ \frac{1}{2}\left(\frac{\omega(\beta^{(1)})^2}{n_1}+\frac{\omega(\beta^{(2)})^2}{n_2}\right)\log{\frac{1}{\delta}}} + 2 R_{n_1,n_2} (l_{\lambda, \gamma})
\end{align*} 
for all $\hat{\beta}$ with probability at least $1-\delta$.
\end{prop}
\begin{proof}
Both of the inequalities can be proved in the same way. We show the first inequality in this proof.
Let $f(\mathcal{D}_1,\mathcal{D}_2) = \sup_{\hat{\beta}} \sum_{m=1}^2\left\{ \hat{L}_{\lambda, \gamma}(\hat{\beta},\beta^{(m)};\mathcal{D}_m) - L_{\lambda, \gamma}(\hat{\beta}, \beta^{(m)}) \right\}$. Then
\begin{align*}
| f(\mathcal{D}_1, \mathcal{D}_2) - f(x^{(1)}_1,y^{(1)}_1,\cdots,x^{(1)\prime}_i,y^{(1)\prime}_i,\cdots,x^{(1)}_{n_1},y^{(1)}_{n_1}, \mathcal{D}_2) |
&\leq \frac{\omega(\beta^{(1)})}{n_1} \text{ and} \\
| f(\mathcal{D}_1, \mathcal{D}_2) - f(\mathcal{D}_1, x^{(2)}_1,y^{(2)}_1,\cdots,x^{(2)\prime}_j,y^{(2)\prime}_j,\cdots,x^{(2)}_{n_2},y^{(2)}_{n_2}) |
&\leq \frac{\omega(\beta^{(2)})}{n_2}
\end{align*}
as before. Hence 
\begin{align*}
\mathbb{P}(f(\mathcal{D}_1, \mathcal{D}_2) - E[f(\mathcal{D}_1,\mathcal{D}_2)] \geq \epsilon) \leq e^{-\frac{2}{A}\epsilon^2}
\end{align*}
for all $\epsilon>0$ where $A = \frac{\omega(\beta^{(1)})}{n_1}+\frac{\omega(\beta^{(2)})}{n_2}$.
The expectation of $f$ is bounded by the Rademacher complexity.
Therefore,
\begin{align*}
\mathbb{P}\left( f(\mathcal{D}_1,\mathcal{D}_2) \geq \epsilon \right) 
&\leq e^{-\frac{2}{A}(\epsilon - 2 R_{n_1,n_2} (l_{\lambda, \gamma}))^2} .
\end{align*}
Set $\delta$ to be the right hand side of the inequality to obtain the result.
\end{proof}

The Rademacher complexity on the joint dataset can be further bounded in the following way.
\begin{prop} \label{rademacher_bound_2}
The Rademacher complexity is bounded by
\begin{align*}
R_{n_1,n_2} (l_{\lambda, \gamma}) 
&\leq \left\{ \frac{1}{8}\sqrt{\frac{3}{\pi}}B \left( \vertiii{\beta^{(1)}_c - \beta^{(2)}_c} 
	+ \vertiii{\beta^{(1)}_c + \beta^{(2)}_c} \right) 
	+ \left(1+\frac{1}{2}\sqrt{\frac{3}{\pi}}\right) \sqrt{\frac{2\log{C}}{\lambda}} BC 
	+ \frac{1}{2}\sqrt{\frac{3}{\pi}}C\log{C} \right\} \\
&\quad \times \sum_{m=1}^2 \frac{1}{\sqrt{n_m}}  
    + \frac{1}{4} \sqrt{\frac{3}{\pi}} B \sum_{m=1}^2 \frac{\vertiii{\beta^{(m)}} }{\sqrt{n_m}} 
\end{align*}
\end{prop}
\begin{proof}
By definition,
\begin{align*}
R_{n_1,n_2} (l_{\lambda, \gamma}) 
&= E\left[\sup_{\hat{\beta}}\sum_{m=1}^2 \frac{1}{n_m} \sum_{i=1}^{n_m} \sigma^{(m)}_i l_{\lambda, \gamma} (\hat{\beta},\beta^{(m)};x^{(m)}_i)\right] \\
&= E \Bigg[\sup_{\hat{\beta}} \frac{1}{n} \sum_{i=1}^n \sigma_i^{(m)} \bigg\{
	\sum_{c=1}^C \lVert \beta_c^{(m)} - \hat{\beta}_c \rVert \lVert x^{(m)}_i \rVert 
	+ \sum_{c=1}^C -\beta_c^{(m)\top} x^{(m)}_i g_\gamma (\beta_c^{(m)\top} x^{(m)}_i) \\
&\quad + \log{\sum_{c=1}^C e^{\hat{\beta}_c^\top x^{(m)}_i}} 
	+ \frac{\lambda}{2} \sum_{c=1}^C \lVert \hat{\beta}_c \rVert^2 \bigg\} \Bigg] \\
&\leq E \left[\sup_{\hat{\beta}} \sum_{m=1}^2 \frac{1}{n_m} \sum_{i=1}^{n_m} \sigma_i^{(m)} \sum_{c=1}^C \lVert \beta_c^{(m)} - \hat{\beta}_c \rVert \lVert x^{(m)}_i \rVert \right] \\
&\quad + E \left[\sup_{\hat{\beta}} \sum_{m=1}^2 \frac{1}{n_m} \sum_{i=1}^{n_m} \sigma_i^{(m)} \sum_{c=1}^C -\beta_c^{(m)\top} x^{(m)}_i g_\gamma (\beta_c^{(m)\top} x^{(m)}_i)  \right] \\
& \quad + E \left[\sup_{\hat{\beta}} \sum_{m=1}^2 \frac{1}{n_m} \sum_{i=1}^{n_m} \sigma_i^{(m)} \log{\sum_{c=1}^C e^{\hat{\beta}_c^\top x^{(m)}_i}}  \right]\\
&\quad + E \left[\sup_{\hat{\beta}} \sum_{m=1}^2 \frac{1}{n_m} \sum_{i=1}^{n_m} \sigma_i^{(m)} \frac{\lambda}{2} \sum_{c=1}^C \lVert \hat{\beta}_c \rVert^2 \right] .
\end{align*}
The first term can be bounded by
\begin{align*}
E \left[\sup_{\hat{\beta}} \sum_{m=1}^2 \frac{1}{n_m} \sum_{i=1}^{n_m} \sigma_i^{(m)} \vertiii{\beta_c^{(m)} - \hat{\beta}_c} \lVert x^{(m)}_i \rVert \right] 
&\leq E \left[\sup_{\hat{\beta}} \sum_{m=1}^2 \frac{1}{n_m} \sum_{i=1}^{n_m} \sigma_i^{(m)} B \vertiii{\beta_c^{(m)} - \hat{\beta}_c} \right] \\
&\leq E \left[\sup_{\hat{\beta}} \sum_{m=1}^2 \frac{1}{n_m} \sum_{i=1}^{n_m} \sigma_i^{(m)}  B \vertiii{\beta_c^{(m)} - \hat{\beta}_c} \mathbbm{1}_{\sum_{i=1}^{n_1} \sigma_i^{(1)}>0, \sum_{i=1}^{n_2} \sigma_i^{(2)}>0} \right] \\
&\quad + E \left[\sup_{\hat{\beta}} \sum_{m=1}^2 \frac{1}{n_m} \sum_{i=1}^{n_m} \sigma_i^{(m)} B \vertiii{\beta_c^{(m)} - \hat{\beta}_c} \mathbbm{1}_{\sum_{i=1}^{n_1} \sigma_i^{(1)}>0, \sum_{i=1}^{n_2} \sigma_i^{(2)}<0}  \right] \\
&\quad + E \left[\sup_{\hat{\beta}} \sum_{m=1}^2 \frac{1}{n_m} \sum_{i=1}^{n_m} \sigma_i^{(m)} B\vertiii{\beta_c^{(m)} - \hat{\beta}_c} \mathbbm{1}_{\sum_{i=1}^{n_1} \sigma_i^{(1)}<0, \sum_{i=1}^{n_2} \sigma_i^{(2)}>0} \right] \\
&\leq E \left[\sup_{\hat{\beta}} \sum_{m=1}^2 \frac{1}{n_m} \sum_{i=1}^{n_m} \sigma_i^{(m)} B \vertiii{\beta_c^{(m)} - \hat{\beta}_c} \mathbbm{1}_{\sum_{i=1}^{n_1} \sigma_i^{(1)}>0, \sum_{i=1}^{n_2} \sigma_i^{(2)}>0} \right] \\
&\quad + \frac{1}{4}\sqrt{\frac{3}{\pi}} \sum_{m=1}^2 \frac{1}{\sqrt{n_m}} B \left( \sqrt{\frac{2\log{C}}{\lambda}}C + \vertiii{\beta_c^{(m)} - \hat{\beta}_c} \right)
\end{align*}
where the last inequality comes from the proof for Rademacher complexity bound on a single dataset and by the fact that $\mathbb{P}(\sum_{i=1}^{n_m} \sigma_i^{(m)}<0)\leq\frac{1}{2}$ for $m=1,2$.
Moreover,
\begin{align*}
E & \left[\sup_{\hat{\beta}}  \sum_{m=1}^2 \frac{1}{n_m} \sum_{i=1}^{n_m} \sigma_i^{(m)} B \sum_{c=1}^C \lVert \beta_c^{(m)} - \hat{\beta}_c \rVert \mathbbm{1}_{\sum_{i=1}^{n_1} \sigma_i^{(1)}>0, \sum_{i=1}^{n_2} \sigma_i^{(2)}>0} \right] \\
&\leq E \left[\sup_{\hat{\beta}} \sum_{m=1}^2 \frac{B}{n_m} \sum_{i=1}^{n_m} \sigma_i^{(m)} 
	\sum_{c=1}^C  \left( \left\lVert \frac{\beta_c^{(1)}+\beta_c^{(2)}}{2} - \hat{\beta}_c \right\rVert 
	+ \left\lVert\frac{\beta_c^{(1)}-\beta_c^{(2)}}{2}\right\rVert \right)
	\mathbbm{1}_{\sum_{i=1}^{n_1} \sigma_i^{(1)}>0, \sum_{i=1}^{n_2} \sigma_i^{(2)}>0} \right]\\
& \leq E \left[\sup_{\hat{\beta}} \sum_{m=1}^2 \frac{B}{n_m} \sum_{i=1}^{n_m} \sigma_i^{(m)} 
	\sum_{c=1}^C  \left\lVert \frac{\beta_c^{(1)}+\beta_c^{(2)}}{2} - \hat{\beta}_c \right\rVert \mathbbm{1}_{\sum_{i=1}^{n_1} \sigma_i^{(1)}>0, \sum_{i=1}^{n_2} \sigma_i^{(2)}>0} \right] \\
&\quad + \sum_{c=1}^C \frac{\lVert\beta_c^{(1)}-\beta_c^{(2)}\rVert}{2} \sum_{m=1}^2  
	E \left[\frac{B}{n_m} \sum_{i=1}^{n_m} \sigma_i^{(m)} \mathbbm{1}_{\sum_{i=1}^{n_1} \sigma_i^{(1)}>0, \sum_{i=1}^{n_2} \sigma_i^{(2)}>0} \right] \\
& \leq \sum_{c=1}^C \left( \frac{\lVert\beta_c^{(1)}+\beta_c^{(2)}\rVert}{2} + \sqrt{\frac{2\log{C}}{\lambda}}\right)   
	\frac{1}{2} \sum_{m=1}^2  E \left[\frac{B}{n_m} \sum_{i=1}^{n_m} \sigma_i^{(m)} 
	\mathbbm{1}_{\sum_{i=1}^{n_m} \sigma_i^{(m)}>0}  \right] \\
&\quad + \sum_{c=1}^C \frac{\lVert \beta_c^{(1)}-\beta_c^{(2)} \rVert}{2}
	\sum_{m=1}^2 E \left[\frac{B}{n_m} \sum_{i=1}^{n_m} \sigma_i^{(m)} \mathbbm{1}_{\sum_{i=1}^{n_1} \sigma_i^{(1)}>0, \sum_{i=1}^{n_2} \sigma_i^{(2)}>0} \right] \\
\end{align*}
and
\begin{align*}
\sum_{m=1}^2 E \left[\frac{B}{n_m} \sum_{i=1}^{n_m} \sigma_i^{(m)} \mathbbm{1}_{\sum_{i=1}^{n_1} \sigma_i^{(1)}>0} \mathbbm{1}_{\sum_{i=1}^{n_2} \sigma_i^{(2)}>0} \right] 
&\leq \sum_{m=1}^2 \frac{1}{2}\frac{B}{n_m}  E \left[\sum_{i=1}^{n_m} \sigma_i^{(m)} \mathbbm{1}_{\sum_{i=1}^{n_m} \sigma_i^{(m)}>0} \right] \\
&\leq \sum_{m=1}^2 \frac{1}{2}\sqrt{\frac{3}{4\pi n_m}}B.
\end{align*}
Hence we get
\begin{align*}
E \left[\sup_{\hat{\beta}} \sum_{m=1}^2 \frac{1}{n_m} \sum_{i=1}^{n_m} \sigma_i^{(m)} \sum_{c=1}^C \lVert \beta_c^{(m)} - \hat{\beta}_c \rVert \lVert x^{(m)}_i \rVert \right] 
&\leq \frac{1}{4}\sqrt{\frac{3}{\pi}} \sum_{m=1}^2 \frac{1}{\sqrt{n_m}} B \left( \sqrt{\frac{2\log{C}}{\lambda}}C + \sum_{c=1}^C\lVert \beta_c^{(m)} \rVert \right) \\
&\quad + \sum_{m=1}^2 \frac{1}{4}\sqrt{\frac{3}{\pi}}\frac{B}{\sqrt{n_m}} \sum_{c=1}^C \left(\frac{\lVert\beta_c^{(1)}-\beta_c^{(2)}\rVert}{2} + \frac{\lVert\beta_c^{(1)}+\beta_c^{(2)}\rVert}{2}\right)\\
&\quad + \frac{1}{4}\sqrt{\frac{3}{\pi}}BC\sqrt{\frac{2\log{C}}{\lambda}}\sum_{m=1}^2 \frac{1}{\sqrt{n_m}}
\end{align*}
The rest of the terms are bounded in the same way as in one dataset case. Adding all the bounds up concludes the proof.
\end{proof}

Subgradient method estimator on the combined dataset also has a similar error bound.
\begin{prop} \label{classification_subgradient_2}
Let the subgradient method estimator on $\mathcal{D}_1 \cup \mathcal{D}_2$ be iteratively defined by 
$\hat{\beta}^{[k+1]} = \hat{\beta}^{[k]} - \eta^{[k]} g^{[k]}$ where $\hat{\beta}^{[1]}$ is an initial estimator, $\eta^{[k]}$ is the $k$th step size, and 
$g^{[k]} \in \partial \sum_{m=1}^2 \hat{L}_{\lambda, \gamma}(\hat{\beta}^{[k](m)}, \beta^{(m)}; \mathcal{D}_m)$. 
Then
\begin{align*}
\min\{\sum_{m=1}^2\hat{L}_{\lambda,\gamma}(\hat{\beta}^{[k](m)};\mathcal{D}_m) | k=1, \cdots, K\} 
- \sum_{m=1}^2 \hat{L}_{\lambda,\gamma}(\hat{\beta}^{*(1,2)};\mathcal{D}_m) 
\leq \frac{8\log{C} + 2G \lambda \sum_{k=1}^K \eta^{[k]2}}{\lambda\sum_{k=1}^K \eta^{[k]}}
\end{align*}
where $\hat{\beta}^{*(1,2)} = \argmin{\hat{\beta}} \sum_{m=1}^2 \hat{L}_{\lambda,\gamma}(\hat{\beta},\beta^{(m)};\mathcal{D}_m)$ is the empirical risk minimizer on the combined dataset.
\end{prop}

Integrating all the propositions gives the following guarantee on the combined dataset. 
\begin{theorem}
\label{thm_classification_2}
Let $\hat{\beta}^{[k]}$ be the subgradient method estimator of $\sum_{m=1}^2 \hat{L}_{\lambda, \gamma}(\hat{\beta}, \beta^{(m)}; \mathcal{D}_m)$. Let $s = \argmin{k \leq K}{\sum_{m=1}^2 \hat{L}_{\lambda,\gamma}(\hat{\beta}^{[k]};\mathcal{D}_m)}$ denote the step that achieved the lowest loss during all the subgradient steps. Then
\begin{align*}
\sum_{m=1}^2 L_\lambda(\hat{\beta}^{[s]};\beta^{(m)}) 
&\leq  \Psi(n_1, n_2, \beta^{(1)}, \beta^{(2)}; \delta) 	
	+ \sum_{m=1}^2L_{\lambda, \gamma}(\beta^{*(1,2)},\beta^{(m)})
\end{align*} 
with probability at least $1-2\delta$ where 
$\Psi = O\left(\left(\sqrt{\frac{1}{n_1}} + \sqrt{\frac{1}{n_2}} \right) \sqrt{\log{\frac{1}{\delta}}} \right) $
and
\begin{align*}
\beta^{*(1,2)} = \argmin{\hat{\beta}} \sum_{m=1}^2 L_{\lambda, \gamma}(\hat{\beta},\beta^{(m)})
\end{align*} 
is the minimizer for the population surrogate loss on the combined dataset.
\end{theorem}
\begin{proof}
Define
\begin{align*}
\Psi(n_1, n_2, \beta^{(1)}, \beta^{(2)}; \delta)
	&=\psi_1 (n_1,n_2,\beta^{(1)}, \beta^{(2)}; \delta)
	+ \psi_2 (n_1, n_2, \beta^{(1)}, \beta^{(2)})
	+ \psi_3
\end{align*}
where
\begin{align*}
\psi_1 (n_1,n_2,\beta^{(1)}, \beta^{(2)}; \delta)
	&= \sqrt{ 2 \left( \frac{\omega(\beta^{(1)})^2}{n_1} + \frac{\omega(\beta^{(2)})^2}{n_2} \right)\log{\frac{1}{\delta}}}, \\
\psi_2 (n_1, n_2, \beta^{(1)}, \beta^{(2)})
	&= \left(\frac{1}{\sqrt{n_1}} + \frac{1}{\sqrt{n_2}}\right)
	\left\{ b_1\left(\vertiii{\beta^{(1)}-\beta^{(2)}} + \vertiii{\beta^{(1)}+\beta^{(2)}}\right) +b_2 \right\} \\
	&\quad + b_3 \left(\frac{\vertiii{\beta^{(1)}}}{\sqrt{n_1}} + \frac{\vertiii{\beta^{(2)}}}{\sqrt{n_2}}\right), \\
\psi_3
	&= \frac{8\log{C} + 2G \lambda \sum_{k=1}^K \eta^{[k]2}}{\lambda\sum_{k=1}^K \eta^{[k]}},
\end{align*}
for constants 
$
b_1 = \frac{1}{2}\sqrt{\frac{3}{\pi}}B,
$
$
b_2 = \left(4 + 2 \sqrt{\frac{3}{\pi}}\right) BC \sqrt{\frac{2\log{C}}{\lambda}} + 2\sqrt{\frac{3}{\pi}}C\log{C}
$, and
$
b_3 = \sqrt{\frac{3}{\pi}}B
$.
Then the theorem can be proved by using the same line of reasoning as in Theorem \ref{thm_classification_1}.
\end{proof}

As before, the bound $\Psi$ consists of three parts. The first part $\psi_1$ is the error stemming from the randomness of the samples, and is connected to $\delta$. The second part $\psi_2$ is the error related to the complexity of the hypothesis class. $\psi_2$ has some overlapping terms that appeared in Theorem \ref{thm_classification_1} as well as new terms that reveal the effect of relationship between $\beta^{(1)}$ and $\beta^{(2)}$. The last part $\psi_3$ is the error due to the subgradient method.

Once again, the theorem displays the effect of size of the datasets and the hypothesis class on the performance of the subgradient method estimator $\sum_{m=1}^2 L_\lambda(\hat{\beta}^{[s]};\beta^{(m)})$. Suppose the step size satisfies conditions as before, which implies that the subgradient method converges. The error bound decreases if there are more samples in $\mathcal{D}^{(1)}$ and $\mathcal{D}^{(1)}$. Similar interpretation can be made on the size of the hypothesis class. 

A new implication from this theorem is that relationship between the true parameters also influences the error bound. The error bound decreases if the true parameters are similar. Indeed, if $\beta^{(1)}$ and $\beta^{(2)}$ are close to each other, then $\vertiii{\beta^{(1)} - \beta^{(2)}}$ is small, reducing $\psi_{2,3}$. It aligns with the intuition that merging the datasets would undoubtedly be advantageous if they are sampled from the same distribution.

\subsubsection{An Equivalent Condition for Error Bound Reduction} \label{section:classification:conclusion}
We propose to combine two datasets based on the error bounds. Recall from Theorem \ref{thm_classification_1} that $\Phi$ serves as an upper bound of the performance of the model trained on a single dataset . Also, $\Psi$ serves as an upper bound of the performance of the model trained on the joint dataset by  Theorem \ref{thm_classification_2}. Merging the datasets when the sum of $\Phi$'s on each dataset is greater than $\Psi$ would decrease the error bound, leading to a better guarantee of the model. Specifically, let $\hat{\beta}^{[s_m](m)}$ denote the subgradient method estimator that achieves the lowest loss on $\mathcal{D}_m$ alone for $m=1,2$. Similarly, let $\hat{\beta}^{[s_{1,2}](1,2)}$ denote the subgradient method estimator on the joint dataset $\mathcal{D}_1 \cup \mathcal{D}_2$. As in Theorem \ref{thm_classification_2}, let $\beta^{*(1,2)}$ denote the minimizer for the population loss on the joint dataset. Then 
\begin{align*}
L_\lambda(\hat{\beta}^{[s_m](m)};\beta^{(m)}) \leq \Phi(n_m, \beta^{(m)};\delta) + L_{\lambda,\gamma}(\beta^{*(1,2)},\beta^{(m)})
\end{align*} for $m=1,2$ and 
\begin{align*}
\sum_{m=1}^2 L_\lambda(\hat{\beta}^{[s_{1,2}](1,2)};\beta^{(m)}) \leq \Psi(n_1,n_2, \beta^{(1)},\beta^{(2)};\delta) + \sum_{m=1}^2 L_{\lambda,\gamma}(\beta^{*(1,2)},\beta^{(m)}).
\end{align*}
Hence the datasets should be combined if $\Psi(n_1,n_2, \beta^{(1)},\beta^{(2)};\delta) \leq \sum_{m=1}^2 \Phi(n_m, \beta^{(m)};\delta)$.
The following theorem states an equivalent condition for the comparison.

\begin{theorem} \label{thm:classification_final}
Let $\Phi$ and $\Psi$ be the functions defined in Theorem \ref{thm_classification_1} and \ref{thm_classification_2}, respectively. Then the error bound of the model trained on the combined dataset is smaller, or equivalently, 
\begin{align*}
\Psi(n_1,n_2, \beta^{(1)},\beta^{(2)};\delta) \leq \sum_{m=1}^2 \Phi(n_m, \beta^{(m)};\delta)
\end{align*} if and only if
\begin{align*}
& \sqrt{ 2\log{\frac{1}{\delta}} \left(\frac{\omega(\beta^{(1)})^2}{n_1}+\frac{\omega(\beta^{(2)})^2}{n_2}\right)} + \sqrt{\frac{3}{\pi}}B \left(\frac{1}{\sqrt{n_1}}+\frac{1}{\sqrt{n_2}}\right) \left(\frac{\vertiii{\beta^{(1)} - \beta^{(2)}}}{2} + \frac{\vertiii{\beta^{(1)} + \beta^{(2)}}}{2}\right) \\
&\leq \sqrt{2\log{\frac{1}{\delta}}}\left(\sqrt{\frac{\omega(\beta^{(1)})^2}{n_1}}+\sqrt{\frac{\omega(\beta^{(2)})^2}{n_2}}\right)
	+ \sqrt{\frac{3}{\pi}} B \left(\frac{\vertiii{\beta^{(1)}}}{\sqrt{n_1}} + \frac{\vertiii{\beta^{(2)}}}{\sqrt{n_2}} \right)
	+ \frac{4\log{C}}{\lambda \sum_{k=1}^K \eta^{[k]}} .
\end{align*}
\end{theorem}

The above condition agrees with what one would intuitively expect when combining two datasets. Suppose the true parameters $\beta^{(1)}$ and $\beta^{(2)}$ are are close to each other. Then the distance $\vertiii{ \hat{\beta}^{(1)} - \hat{\beta}^{(2)}}$ is small, decreasing the left-hand-side of the inequality. In turn, the above condition would possibly hold, which further indicates that the datasets should be combined.

\section{Experimental Details} \label{comparisonsection}
In this section, we provide details for the implementation of the neural network model, and compare the algorithm with other clustering method for multitask learning. The method that is most relevant to our study is BiFactor MTL proposed by \citet{murugesan2017co}. 

\subsection{Neural Network Setup}\label{app_nn}

We employ two simple neural network architectures: a single hidden layer multi-layer perceptron (MLP1) and a two hidden layer multi-layer perceptron (MLP2). The following details outline the configurations:

\begin{itemize} \item \textbf{Network Architecture}: The input datasets consist of Boom Bikes (BB) \citep{mishra2021}, Demand Forecast for Optimized Inventory Planning (DFOIP) \citep{aguilar2023}, and Walmart Data Analysis (WDAF) \citep{sahu2023}, with input feature dimensions of 7, 9, and 12, respectively. As datasets have different number of features, models are built based on the input dimension accordingly. In both MLP1 and MLP2, the hidden layers have the same dimension as the input feature dimension of each input.

\item \textbf{Training the Representation}: For each dataset, we train the models using the entire dataset to enable the capture of more complex patterns and generate richer feature representations. It is important to note that this step is not intended to produce the most accurate predictive model but rather to focus on learning a more informative representation of the data for subsequent tasks. The models are trained with a batch size of 64 over 10 epochs, with a learning rate of 1e-4 and weight decay of 0.9. The Mean Squared Error (MSE) loss function is used, optimized by stochastic gradient descent. Training was conducted on a single Nvidia RTX 6000 (Ampere Version) GPU.

\item \textbf{Integration with Algorithm \ref{alg2}}: Let $f(\cdot)$ and $g(\cdot)$ represent the trained MLP1 and MLP2 models, respectively. Both models consist of hidden layers followed by a linear layer, expressed as $f(\cdot) = l \circ f_{h}(\cdot)$ and $g(\cdot) = l \circ g_{h}(\cdot)$, where $f_{h}(\cdot)$ and $g_{h}(\cdot)$ represent the hidden layer functions, and $l(\cdot)$ represents the linear layer. For datasets $\{(x_i, y_i)\}_{i=1}^n$, we transform the input into $\{(e_i, y_i)\}_{i=1}^n$, where $e_i$ is the representation from the penultimate layer. That is, either $e_i = f_{h}(x_i)$ or $e_i = g_{h}(x_i)$, depending on the model we use. Algorithm \ref{alg2} can then be applied to the transformed dataset $\{(e_i, y_i)\}_{i=1}^n$, as the linear layer structure aligns perfectly with the regression setting.
\end{itemize}

\subsection{Choice of Hyperparameters}\label{section:hyperparameter_setup}
We further provide more details on the choice of hyperparameters of our algorithms.
\begin{itemize}
    \item \textbf{Grid search of $\alpha$ in Algorithm \ref{alg1}}: The grid search of $\alpha$ in Part1 is applied over the interval $[\alpha_{min}, \alpha_{max}] = [2, 10]$ with window size $\eta=0.01$.
    \item \textbf{Threshold $\lambda$ in Algorithm \ref{alg1}} and \ref{alg2}: $\lambda$ is fixed as $0.9$ throughout all experiments.
    \item \textbf{The number of iterations in Algorithm \ref{alg1}}: In Part1, $max_iterations$ is set as $1000$.
    \item \textbf{Size of training dataset}: We chose $n_1$ and $n_2$ as $50$ or $100$, depending on the feature size $p$.
    \item \textbf{The number of out-of-samples}: Across all experiments, the number of out-of-samples is fixed as $1000$.
    \item \textbf{Performance comparison in Section \ref{sec_emp_syn}}: The columns \texttt{Algorithm \ref{alg1}} and \texttt{Direct Comparison} in  Table \ref{table1} and Table \ref{table1.2} are obtained from $1000$ repetitions.
\end{itemize}

\subsection{Comparison with other models}
To the best of our knowledge, no existing work specifically focuses on combining datasets with the goal of minimizing out-of-sample error. However, the algorithm developed by \citet{murugesan2017co} could serve as an alternative approach. In \citet{murugesan2017co}, factor models are employed to decompose the weight matrix into feature clusters and task clusters. To facilitate a comparison between the two methods, we apply each approach to the features obtained by applying the MLP1 model to the datasets introduced earlier.

The result of Algorithm \ref{alg2} and BiFactor MTL is summarized in Table \ref{table4}. BB is divided into three parts based on the weather condition, DFOIP is partitioned into four parts based on store types, and WDAF is split into five parts according to store locations. For each dataset, only a subset of the entire dataset was chosen for this additional experiment to expedite the clustering. The reported values of $\widehat{\text{OSE}}$ on individual datasets may be different for the two methods as they both have some randomness when selecting out-of-samples and computing the $\widehat{\text{OSE}}$.

\begin{table}[t]
\caption{Algorithm \ref{alg2} on Representation with $\widehat{\text{OSE}}$}
\label{table4}
    \centering
    \begin{tabular}{@{}c|c|c|c@{}}
        \toprule
                        & BB        & DFOIP                 & WDAF                  \\
        \midrule
        Individual      & $2.822$   & $3.097 \times 10^3$   & $1.495 \times 10^2$   \\
        Alg. \ref{alg2} & $2.359$   & $9.105 \times 10^2$   & $3.562 \times 10^1$   \\
        Reduction       & $16.41\%$ & $70.60\%$             & $76.17\%$             \\
        \bottomrule
    \end{tabular}
    
    \begin{tabular}{@{}c|c|c|c@{}}
        \toprule
                        & BB        & DFOIP                 & WDAF \\
        \midrule
        Individual      & $3.308$   & $2.908 \times 10^3$   & $1.114 \times 10^2$   \\
        BiFactor        & $3.239$   & $1.167 \times 10^3$   & $9.697 \times 10^1$   \\
        Reduction       & $2.086\%$ & $59.87\%$             & $12.90\%$             \\
        \bottomrule
    \end{tabular}
\end{table}

As shown in Table \ref{table4}, Algorithm \ref{alg2} significantly reduces $\widehat{\text{OSE}}$, whereas the reduction for the BiFactor MTL method is less significant. This discrepancy in $\widehat{\text{OSE}}$ reduction is attributed to the different dataset combinations selected by each algorithm. For example, Algorithm \ref{alg2} opted to combine the datasets for weather conditions 1 and 2 while keeping condition 3 separate on BB, whereas BiFactor MTL merged the datasets for weather conditions 2 and 3. Our algorithm effectively identifies relevant datasets for combination, leading to lower errors on the merged datasets.

\subsection{Complexity Analysis}\label{section:complexity}
Algorithm \ref{alg2} uses pairwise comparison, which effectively reduces the running time by having a quadratic growth on the number of datasets. As our method employs linear models, the actual execution time does not extend uncontrollably when the number of datasets scales. For instance, it took 10 minutes for RSS and 3 hours for CFGSF to train an MLP and apply Algorithm \ref{alg2}, on a standard CPU without parallelization. 

To our knowledge, there is no off-the-shelf clustering method applicable to our setting to effectively reduce the computational complexity due to the following reason. Most clustering methods commonly assume that objects to be clustered have representative features. However, our setting only has a notion of pairwise similarity between two datasets and no universal feature space of the datasets themselves, which renders those clustering algorithms inapplicable. Other clustering methods that operate without features usually require a similarity matrix which contains information about similarity between any two objects. However, the complexity of computing the similarity matrix is tantamount to that of Algorithm \ref{alg2} in the worst case. Hence we believe that applying the greedy algorithm not only simplifies the clustering, which is not the most core part of our method, but also potentially reduces the computational cost. Given the computational resources commonly available today, we do not view computational complexity as a significant limitation. Developing a clustering-based method within our framework would require substantial additional work, and we hope that the exploration of clustering strategies could be done in future research.

\end{document}